\documentclass[11pt]{amsart}
\usepackage{geometry}                % See geometry.pdf to learn the layout options. There are lots.
\usepackage{algorithm}
\usepackage{algorithmic}
\usepackage{units}
\RequirePackage{natbib}

\usepackage{graphicx}
\usepackage{amssymb}
\usepackage{epstopdf}
\usepackage{amsaddr}
\newtheorem{definition}{Definition}
\newtheorem{theorem}{Theorem}
\newtheorem{corollary}{Corollary}
\newtheorem{lemma}{Lemma}
\newtheorem{prop}{Proposition}

\title{The blessing of transitivity in sparse and stochastic networks}%; a triptych on triples}
\author{Karl Rohe, Tai Qin}
\address{Department of Statistics, UW-Madison}
\email{name[at]stat.wisc.edu, for karlrohe or qin}
\begin{document}
\let\thefootnote\relax\footnote{Research supported in part by NIH Grant EY09946, NSF Grant DMS-0906818, and grants from WARF.}
\maketitle
\begin{abstract}
The interaction between transitivity and sparsity, two common features in empirical networks, implies that there are local regions of large sparse networks that are dense. We call this the blessing of transitivity and it has consequences for both modeling and inference. Extant research suggests that statistical inference for the Stochastic Blockmodel is more difficult when the edges are sparse. However, this conclusion is confounded by the fact that the asymptotic limit in all of the previous studies is not merely sparse, but also non-transitive.  To retain transitivity, the blocks cannot grow faster than the expected degree.  Thus, in sparse models, the blocks must remain asymptotically small.

Previous algorithmic research demonstrates that small ``local" clusters are more amenable to computation, visualization, and interpretation when compared to  ``global" graph partitions.  This paper provides the first statistical results that demonstrate how these small transitive clusters are also more amenable to statistical estimation.  Theorem 2 shows that a ``local" clustering algorithm can, with high probability, detect a transitive stochastic block of a fixed size (e.g. 30 nodes) embedded in a large graph.  The only constraint on the ambient graph is that it is large and sparse--it could be generated at random or by an adversary--suggesting a theoretical explanation for the robust empirical performance of local clustering algorithms. 
\end{abstract}

\section*{Introduction}

Advances in information technology have generated a barrage of data on highly complex systems with interacting elements.  Depending on the substantive area, these interacting elements could be metabolites, people, or computers. Their interactions could be represented in chemical reactions, friendship, or some type of communication.  Networks (or graphs) describe these relationships.  Therefore, the questions about the relationships in these data are questions regarding the structure of networks.  Several of these questions are more naturally phrased as questions of inference; they are questions not just about the realized network, but about the mechanism that generated the network.  To study questions in graph inference, it is essential to study algorithms under model parameterizations that reflect the fundamental features of the network of interest.  

Sparsity and transitivity are two fundamental and recurring features.  In sparse graphs, the number of edges in a network is orders of magnitude smaller than the number of possible edges; the average element has 10s or 100s of relationships, even in networks with millions of other elements. 
 Transitivity describes the fact that friends of friends are likely to be friends.  The interaction of these simple and localized features has profound implications for determining the set of realistic statistical models.  They imply that in large sparse graphs, there are local and dense regions.  This is the blessing of transitivity.

One essential inferential goal is the discovery of communities or clusters of highly connected actors.  These form essential feature in a multitude of empirical networks,
%\citep{lesk}  
and identifying these clusters helps answer vital scientific  questions in many fields.  A terrorist cell is a cluster in the communication network of terrorists; web pages that provide hyperlinks to each other form a community that might host discussions of a similar topic; a cluster in the network of biochemical reactions might contain metabolites with similar functions and activities.   Several papers, that are briefly reviewed below, have proved theoretical results for various graph clustering algorithms under the Stochastic Blockmodel, a parametric model, where the model parameters correspond to a true partition of the nodes.  Often, these estimators are also studied under the exchangeable random graph model, a non-parametric generalization of the Stochastic Blockmodel.  The overarching goal of this paper is to show   (1)  how sparse and transitive models require a novel asymptotic regime and (2) how the blessing of transitivity makes edges become more informative in the asymptote, allowing for statistical inference even when  cluster size and expected degrees do not grow with the number of nodes.

The first part of this paper studies how  sparsity and transitivity interact in the Stochastic Blockmodel, and more generally, in the exchangeable random graph model.   Interestingly, if a Stochastic Blockmodel is both sparse and transitive, then it has small blocks.  The second part of this paper (1) introduces an intuitive and fast local clustering technique to find small clusters; (2) proposes the local Stochastic Blockmodel, which presumes a single stochastic block is embedded in a sparse and potentially adversarially chosen network; and (3) proves that if the proposed local clustering technique is initialized with any point in the stochastic block, then it returns the block with high probability. Figure \ref{fig:epinionClusters} illustrates the types of clusters found by the proposed algorithm;  in this case from a social network on epinions.com containing over 76,000 people.

% While clusters describe a global feature of the network (i.e. partitioning the nodes into distinct groups), transitivity describes a local manifestation of clustering and communities.  In a transitive network, if nodes $i$ and $j$ are friends and nodes $j$ and $k$ are friends, then it is likely that $i$ and $k$ are friends.  \cite{harary1979matrix} introduced the ``transitivity ratio" as the ratio of number of closed triplets (i.e. triangles) over the number of connected triplets (i.e. two stars) to measure a network's transitivity.
%%\[\mbox{Transitivity Ratio} = \frac{\mbox{number of closed triplets}}{\mbox{number of connected triples of vertices}}.\]
%This measure is similar to the clustering coefficient proposed in \cite{Watts}.  Both can be thought of as estimates of the probability that two nodes will be connected, given that they have a common friend.   In empirical networks, this quantity is typically greater than $.1$ and can be as large as $.6$ or larger \citep{newman2009random}.  In social networks in particular, values between $.3$ to $.6$ are typical (\cite{snijders2011statistical}).  
%%XXX Tai, can you determine how to provide a citation to these lecture notes (see the .tex code, url in a comment).
%%\footnote{\verb"http://www.stats.ox.ac.uk/~snijders/Trans_Triads_ha.pdf", even in massive networks the clustering coefficient is rarely less than .1 \citep{SNAP}.} 

\begin{figure}[htbp] %  figure placement: here, top, bottom, or page
   \centering
   \includegraphics[height=1.62in]{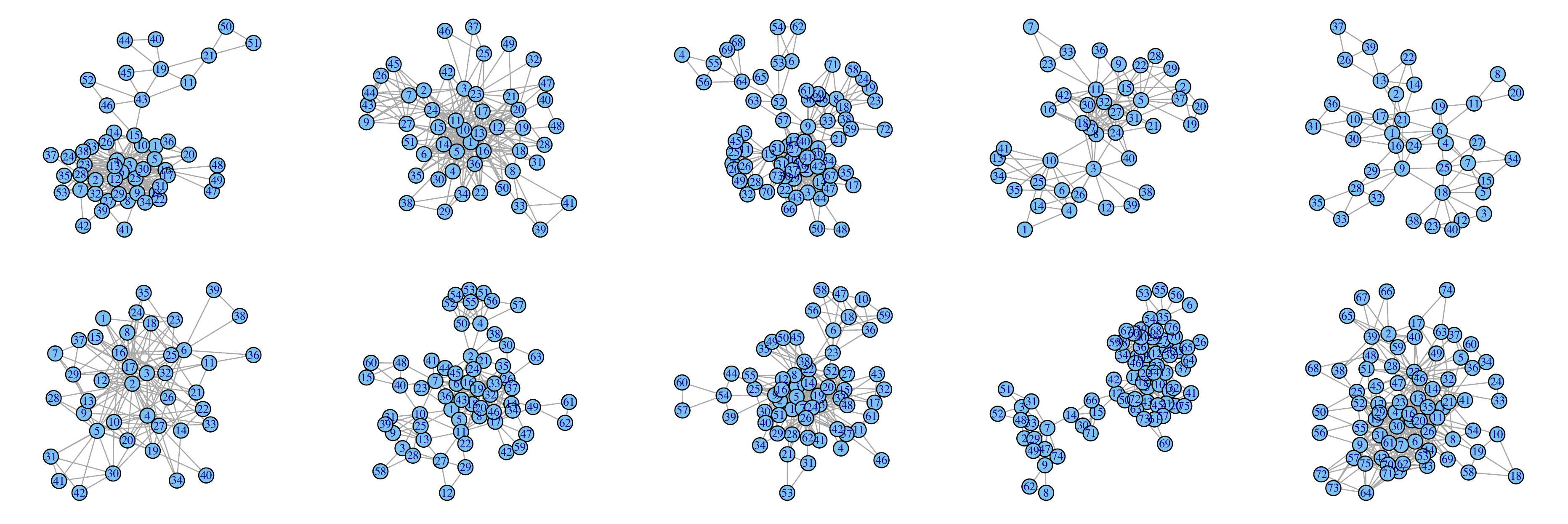} 
   \caption{Local clusters from a sparse 76k node social network from epinions.com.  Created with the igraph library in R \citep{igraph}.  }
   \label{fig:epinionClusters}
\end{figure}

\subsection{Preliminaries}
Networks, or graphs, are represented by a vertex set and an edge set, $G = (V,E)$, where $V = \{1, \dots, n\}$ contains the actors and 
\[E = \{(i,j) : \mbox{ there is an edge from $i$ to $j$}\}.\] 
The edge set can be represented by the adjacency matrix $A \in \{0,1\}^{n \times n}$:
\begin{equation} \label{Adef}
A_{ij} = \left\{\begin{array}{cl}1 & \textrm{ if  $(i,j) \in E$}  \\
0 &  \textrm{ otherwise.}\end{array}\right.
\end{equation}
This paper only considers undirected graphs.  That is, $(i,j) \in E \Rightarrow (j,i) \in E$. The adjacency matrix of such a graph is symmetric.  Many of the results in the paper have a simple extension to weighted graphs, where $A_{ij} \ge 0$.  For simplicity, we only discuss unweighted graphs.  For $i \in V$, let $d_i = \sum_\ell A_{i \ell}$ denote the degree of node $i$.  Define $N_i$ as the neighborhood of node $i$, 
\[N_i = \{j: (i,j) \in E\}.\] 

Define the transitivity ratio of $A$ as 
\[trans(A) = \frac{ \mbox{number of closed triplets in $A$}}{\mbox{number of connected triples of vertices in $A$}}.\]

%Both the numerator and the denominator of the transitivity ratio have other formulations that suggest how they can be  computed.  For example, let $\lambda_1 \ge \dots \ge \lambda_n$ denote the eigenvalues of $A$, then
%\begin{eqnarray*}
%&&\mbox{Number of triangles in $A$} = \tfrac{1}{6} \mbox{trace}(A^3) = \tfrac{1}{6} \sum_j \lambda_j^3\\
%&&\mbox{Number of 2-stars in $A$}= \sum_j {d_j \choose 2}  = \tfrac{1}{2}\sum_j d_j^2 - |E|,
%\end{eqnarray*}
%where $|E|$ is the total number of edges.  
 
 \cite{watts1998collective} introduced an alternative measure of transitivity, the clustering coefficient.  
The local clustering coefficient,  $C(i)$, is defined as the density of the subgraph induced by $N_i$, that is the number of edges between nodes in $N_i$ divided by the total number of possible edges.  The clustering coefficient for the entire network is the average of these values.
 \[C = \frac{1}{n} \sum_i C(i)\]
 This is related to the triangles in the graph because an edge $(j,k)$ between two nodes in $N_i$ makes a triangle with node $i$.  
 
\subsection{Statistical models for random networks}

Suppose $A = \{A_{ij}: i,j\ge 1\}$ is an infinite array that is binary, symmetric, and random. If $A$ is (jointly) exchangeable, that is
\[A \stackrel{d}{=} A_{\sigma, \sigma} = \{A_{\sigma(i), \sigma(j)}: i,j\ge 1\},\]
for any arbitrary permutation $\sigma$, then the Aldous-Hoover representation says there exists $i.i.d.$ random variables $\xi_1, \xi_2, \dots$ and an additional independent random variable $\alpha$ such that conditional on these variables, the elements of $A$ are statistically independent \citep{hoover1979relations, aldous1981representations,kallenberg2005probabilistic}.  
%Define $g$ as
%\begin{equation} \label{erg}
%g(\alpha, \xi_i, \xi_j) = P(A_{ij} = 1| \alpha, \xi_i, \xi_j).
%\end{equation}
The global parameter $\alpha$ controls the edge density and for ease of notation, it is often dropped \citep{bickel2011method}.  One should think of this result as an extension of DeFinetti's Theorem to infinite exchangeable arrays. 

While this representation has only been proven to be equivalent to exchangeability for infinite arrays, it is convenient to adopt this representation for finite graphs. 

\begin{definition}\label{exchmod}   Symmetric adjacency matrix $A \in \{0,1\}^{n \times n}$  follows the \textbf{exchangeable random graph model} if there exists i.i.d. random variables $\xi_1, \dots, \xi_n$ such that  probability distribution of  $A$ satisfies
\[P(A|\xi_1, \dots, \xi_n) = \prod_{i<j} P(A_{ij} | \xi_i,\xi_j).\]
\end{definition}
For brevity, we will sometimes refer to this as the exchangeable model. 

Independently of the research on infinite exchangeable arrays, \cite{hoff2002latent} proposed the latent space model which assumes that (1) each person has a set of latent characteristics (e.g. past schools, current employer, hobbies, etc.) and (2) it is only these characteristics that produce the dependencies between edges.  Specifically, conditional on the latent space characteristics, the relationships (or lack thereof) are independent.  The Latent Space Model is equivalent to the exchangeable model in Definition \ref{exchmod}.

%The exchangeable random graph model, also referred to as the latent space model, is a general nonparametric model provide a test bed to study the inferential capabilities of various algorithms as statistical estimators.  

The Stochastic Blockmodel is an exchangeable model that was first defined in \cite{hollandkathryn1983stochastic}. 
\begin{definition} 
The \textbf{Stochastic Blockmodel} is an exchangeable random graph model with $\xi_1, \dots, \xi_n \in \{1, \dots, K\}$ and 
\[P(A_{ij} = 1| \xi_i, \xi_j) = \Theta_{\xi_i, \xi_j}\]
for some $\Theta \in [0,1]^{K \times K}$.
\end{definition}
In this model, the $\xi_i$ correspond to group labels.  The diagonal elements of $\Theta$ correspond to the probability of within-block connections.  The off-diagonal elements correspond to the probabilities of between-block connections.  When the diagonal elements are sufficiently larger than the off-diagonal elements, then the sampled network will have clusters that correspond to the blocks in the model.

The next subsection briefly reviews the existing literature that examines the consistency of various estimators for the partition created by the latent variables $\xi_1, \dots, \xi_n$ in the Stochastic Blockmodel.

\subsection{Previous research}\label{previouslit}
This paper builds on an extensive body of literature examining various types of statistical estimators for the latent partition $\xi_1, \dots, \xi_n$ in the Stochastic Blockmodel.  These estimators fall into four different categories.\footnote{The works cited in this section give a sample of the previous literature on statistical inference for the Stochastic Blockmodel; it is not meant to be an exhaustive list.  In particular, there are several highly relevant papers in the Computer Science literatures on (i) the planted partition model and (ii) the planted clique problem.  The curious reader should consult the references in \cite{ames2010convex} and \cite{chen2012clustering}.
}  
\begin{enumerate}
\item Several have studied estimators that are solutions to  discrete optimization problems (e.g.  \cite{bickel2009nonparametric, choi2010stochastic,  zhao2011extraction, zhao2012degree, flynn2012consistent}).  These objective functions are the likelihood function for the Stochastic Blockmodel or the Newman-Girvan modularity \citep{newman2004finding}, a measure that corresponds to cluster quality.
\item Others have have studied various approximations to the likelihood that lead to more computationally tractable estimators.  For example, \cite{celisse2011consistency} and \cite{bickel2012asymptotic} studied the variational approximation to the likelihood function and \cite{chen2012fitting} studied the maximum pseudo-likelihood estimator.  
\item Building on on spectral graph theoretic results \citep{donath1973lower, fiedler1973algebraic}, several researchers have studied the statistical performance of spectral algorithms for estimating the partition in the Stochastic Blockmodel \citep{mcsherry2001spectral, dasgupta2004spectral, giesen2005reconstructing, coja2009finding, rohe2011spectral, rohe2012co, chaudhuri2012spectral, jin2012fast, sussman, fishkind2013consistent}.  Others have studied estimators that are solutions to semi-definite programs \citep{ames2010convex, oymak2011finding, chen2012clustering}.
\item More recently, \cite{bickel2011method, channarond2011classification, rohe2012co} have developed methods to stitch together network motifs, or simple ``local" measurements on the network,  in a way that estimates the partition in the Stochastic Blockmodel.  \cite{bickel2011method} draws a parallel between this motif-type of estimator and method of moments estimation.
\end{enumerate}

All of the previous results described above are sensitive to both (a) the population of the smallest block and (b) the expected number of edges in the graph;  larger blocks and higher expected degrees lead to stronger conclusions.  This limitation arrises because the proofs rely on some form of concentration of measure for a function of sufficiently many variables.  Bigger blocks and more edges yield more variables, and thus, more concentration.  This paper shows how transitivity leads to a different type of concentration of measure, where each edge becomes asymptotically more informative.  As such, the results in this paper extend to blocks of fixed sizes and bounded expected degrees.

Section \ref{clustering} proposes the \texttt{LocalTrans} algorithm that exploits the triangles built by network transitivity.  As such, it is most similar to motif-type estimators in bullet (4).  Our analysis of \texttt{LocalTrans} is the first to study whether a local algorithm (i.e. initialized from a single node) can estimate a block in the Stochastic Blockmodel.  The emphasis on local structure aligns with the aims of network scan statistics;  these compute a ``local" statistic on the subgraph induced by $N_i$ for all $i$ and then return the maximum over all $i$.  In the literature on network scan statistics, \cite{rukhin2012limiting} and \cite{scanPriebe} have previously studied the anomaly detection properties under random graph models, including a version of the Stochastic Blockmodel.

\section{Transitivity in sparse exchangeable random graph models} \label{sec:trans}

In this section, Proposition \ref{poptrans} and Theorem \ref{sampleTrans} show that previous parameterizations of the sparse exchangeable models and sparse Stochastic Blockmodels lack transitivity in the asymptote. That is, the sampled networks are asymptotically sparse, but they are not asymptotically transitive. 
Theorem \ref{sbmtrans} concludes the section by describing a parameterizations that produce sparse \textit{and} transitive networks.  

Define 
\begin{equation}\label{pmax}
p_{\max} = \max_{\xi_i, \xi_j}  P(A_{ij} = 1 |\xi_i, \xi_j)
\end{equation} 
as the largest possible probability of an edge under the exchangeable model.  
In the statistics literature, previous parameterizations of sparse Stochastic Blockmodels, and sparse exchangeable models, have all ensured sparsity by sending $p_{\max} \rightarrow 0$.

Define 
\begin{equation}\label{ptri}
p_\Delta = P(A_{uv} = 1| A_{iu} = A_{iv} = 1)
\end{equation}
a population measure of the transitivity in the model.  It is the probability of completing a triangle, conditionally on already having two edges. 

By sending $p_{\max}$ to zero, a model removes transitivity.
\begin{prop} \label{poptrans}
Under the exchangeable random graph model \eqref{exchmod} 
\[p_\Delta \le p_{\max},\]
where these probabilities are defined in equations \eqref{pmax} and \eqref{ptri}.
\end{prop}
The next theorem gives conditions that imply the transitivity ratio of the sampled network converges to zero. 
\begin{theorem}\label{sampleTrans} Under the exchangeable random graph model (Definition \ref{exchmod}), define $\lambda_n = E(d_{i})$ as the expected node degree.  If $\lambda_n\rightarrow \infty$, $\lambda_n = o(n)$, and 
\begin{equation}\label{pmaxfrac}
p_{\max} =  o\left(\frac{P(A_{ij} = 1)} {P(A_{ij} = 1|A_{i\ell} = 1)}\right),
\end{equation}
where $p_{\max}$ is defined in \eqref{pmax}, then $trans(A) \stackrel{P}{\rightarrow} 0$.
\end{theorem}
A proof of this theorem can be found in the appendix.  

%For an alternative interpretation of the bound in Theorem \ref{sampleTrans}, the fraction can be re-expressed in terms of unconditional probabilities.
%\[\frac{P(A_{ij} = 1)} {P(A_{ij} = 1|A_{i\ell} = 1)} = \frac{P(A_{ij} = 1)P(A_{i\ell} = 1)} {P(A_{ij} = 1, A_{i\ell} = 1)}\]

The denominator on the right hand side of Equation \eqref{pmaxfrac} quantifies how many edges are adjacent to the average edge, and thus controls how many 2-stars are in the graph. It can be crudely bounded with the maximum expected degree over the latent $\xi_i$. For
\[\lambda_n^{\max} = \max_{\xi_i} \ E(d_i|\xi_i), \] 
it follows that 
\[ P(A_{ij} = 1|A_{i\ell} = 1) \le \lambda_n^{\max} / n.  \] 

\begin{corollary} \label{pmaxcor}
Under the exchangeable random graph model (Definition \ref{exchmod}), define $\lambda_n = E(d_{i})$ as the expected node degree.  If $\lambda_n\rightarrow \infty$, $\lambda_n = o(n)$, and 
\[p_{\max} =  o\left(\frac{\lambda_n} {\lambda_n^{\max}}\right),\]
then $trans(A) \stackrel{P}{\rightarrow} 0$.
\end{corollary}

So, ensuring sparsity  by sending $p_{\max}$ to zero removes transitivity both from the model and from the sampled network.  The next subsection investigates the implications of restricting $p_{\max}>\epsilon>0$ in sparse networks.

\subsection{Implications of non-vanishing $p_{\max}$ in the Stochastic Blockmodel} 

It is easiest to consider a simplified parameterization of the Stochastic Blockmodel. The following parameterization is also called the planted partition model.
%To most easily see how sparsity and transitivity interact in the Stochastic Blockmodel, define the four parameter Stochastic Blockmodel, 
\begin{definition}\label{foursbm}
The \textbf{four parameter Stochastic Blockmodel} is a Stochastic Blockmodel with $K$ blocks, exactly $s$ nodes in each block, $\Theta_{ii} = p$, and $\Theta_{ij} =r$ for $i\ne j$.  
\end{definition}
In this model, (1) $n= Ks$, (2) the ``in-block" probabilities are equal to $p$, and (3) the ``out-of-block" probabilities equal to $r$.   Moreover, the expected degree\footnote{For ease of exposition, this formula allows self-loops.}  of each node is 
\begin{equation}\label{avgdeg}
\mbox{Expected degree under the four parameter model} = sp + (n-s)r.
\end{equation}
Define this quantity as $\lambda_n$.  Under the four parameter model, $p$ is analogous to $p_{\max}$.  %However,  the four parameter model conditions on the latent variables $\xi_i$ (to ensure each block has the same population).  As such, $p$ is 

Note that $sp\le \lambda_n$ and
\[n \frac{p}{\lambda_n} \le K.\]
In sparse  and transitive graphs, $\lambda_n$ is bounded and $p$ is  non-vanishing.    In this regime, $K$ grows proportionally to $n$. The following proposition states this fact in terms of $s$, the population of each block.
% does not grow nearly proportionally to $n$, then the probability of an in-block connection $p$ converges to zero.  Specifically, $K = o(n/d(n))$ implies that    $p \rightarrow 0$.
\begin{prop}
Under the four parameter Stochastic Blockmodel, if $p$ is bounded from below, then
\[s =O(\lambda_n)\]
where $s$ is the population of each block and $\lambda_n$ is the expected node degree.
\end{prop}
%In the four parameter Stochastic Blockmodel, the expected degree is greater than $sp$.  If one wishes to bound the expected degree by $d(n)$, then $s \le d(n)/p$.  
%So, $d(n)$ restricts the rate at which $s$ can grow.  Namely, if $d(n)$ is bounded, then so is $s$, the size of each block. Because $n = Ks$, the slowest possible rate for the growth of $K$ is $n/d(n)$.  When the graph is sparse, this implies that $K$ is growing proportional to, or nearly proportionally to $n$. 

The following Theorem shows that graphs sampled from this parameterization are asymptotically transitive.
\begin{theorem} \label{sbmtrans}
%XXX Tai, can you put the HDSBM trans theorem here?  Here is a bit to get you started... XXX
Suppose that $A$ is the adjacency matrix sampled from the four parameter Stochastic Blockmodel (Definition \ref{foursbm}) with $n$ nodes.  If $p >\epsilon>0, \ r = O(n^{-1}),$ and $s \ge 3$, then as $n \rightarrow \infty$
\[trans(A) \stackrel{P}{\rightarrow} c >0,\]
where $c$ is a constant which depends on $p, r, s$.
\end{theorem}
%XXXX Tai, I can't remember where we left this discussion.  Is $s\ge 3$ sufficient? XXX
\textbf{Remark:} If $r = \frac{c_0}{n}$, for some constant $c_0$, then 
\[ c \approx \frac{p^3s^2}{p^2s^2+c_0^2+2spc_0}.\]  
%XXX Tai, is this correct?  I changed the wording a bit... XXX

%Tai, I presume that we can tolerate $s$ growing.  So, make sure the conditions allow this.  I presume that the main condition is something like the expected number of out-of-block edges is asymptotically proportional to the expected number of in-block edges.

%This Theorem shows that there exist Stochastic Blockmodels that are asymptotically sparse and transitive.  Note that this proposition allows for the scenario when the expected degree is not growing asymptotically. 
The appendix contains the proof for Theorem \ref{sbmtrans}. %XX put in appendix XX. 

%In the planted clique problem in the computer science literature, there is set of $s$ nodes that are fully connected and embedded in a network of $n$ nodes where every edge not in the clique appears independently with probability $1/2$ \citep{ames2010convex}.  Under this planted clique problem, no one has found a polynomial time algorithm that can discover the planted clique with high probability when $s = o(\sqrt{n})$.  This makes one doubt whether the stochastic blocks are discoverable when $s$ is fixed.  
%
%XX how to link, planted clique, empricial lesk dunbar, and rmle motivated by lesk dunbar XXX?

The fixed block size asymptotics in Theorem \ref{sbmtrans} align with two pieces of  previous empirical research suggesting the ``best" clusters in massive networks are small.  \cite{leskovec2009community}  found that in a large corpus of empirical networks, the tightest clusters (as judged by several popular clustering criteria) were no larger than 100 nodes, even though some of the networks had several million nodes. This result is consistent with findings in Physical Anthropology. \cite{dunbar1992neocortex} took various measurements of brain size in 38 different primates and found that the size of the neocortex divided by the size of the rest of the brain had a log-linear relationship with the size of the primateÕs natural communities. In humans, the neocortex is roughly four times larger than the rest of the brain. Extrapolating the log-linear relationship estimated from the 38 other primates, \cite{dunbar1992neocortex} suggests that the average human does not have the social intellect to maintain a stable community larger than roughly 150 people (colloquially referred to as DunbarÕs number). \cite{leskovec2009community} found a similar result in several other networks that were not composed of humans. The research of \cite{leskovec2009community} and \cite{dunbar1992neocortex} suggests that the block sizes in the Stochastic Blockmodel should not grow asymptotically. Rather, block sizes should remain fixed (or grow very slowly).

\subsection{Implications for the exchangeable model} The interaction between sparsity and transitivity also has surprising implications in the more general exchangeable random graph model.  To see this, note that in the exchangeable model,  it is sufficient to assume that $\xi_1, \dots, \xi_n$ are $i.i.d.$ Uniform$(0,1)$ \citep{kallenberg2005probabilistic}.  Then, the conditional density of $\xi_i$ and $\xi_j$ given $A_{ij}= 1$ is 
\begin{equation}\label{cond}
c(\xi_i, \xi_j) = \frac{P(A_{ij} = 1| \xi_i, \xi_j)}{P(A_{ij} = 1)}  
\end{equation}
When $p_{\max}$ does not converge to zero, there exist values of $\xi_i^*$ and $\xi_j^*$ such that $P(A_{ij} = 1| \xi_i, \xi_j)$ does not converge to zero.  However, in a sparse graph, the edge density $P(A_{ij} = 1)$ (in the denominator of Equation \eqref{cond}) converges to zero.  So, $c(\xi_i^*, \xi_j^*)$ is asymptotically unbounded.  For example, in the popular $P(A_{ij} = 1) = O(1/n)$ limit, $c(\xi_i^*, \xi_j^*)$ is proportional to $n$.  In a sense, as a sparse and transitive network grows, each edge becomes more informative.  This is the blessing of transitivity in sparse and stochastic networks.

This asymptotic setting, where $p_{\max}$ is bounded from below, makes for an entirely different style of asymptotic proof; the asymptotic power comes from the fact that each edge becomes increasingly informative in the asymptote.  Previous consistency proofs rely on concentration of measure for functions of several independent random variables (i.e. several edges).  In the sparse and transitive asymptotic setting, concentration follows from the blessing of transitivity, allowing asymptotic results with fixed block sizes and bounded degrees.  For example, in Theorems \ref{Aclust} and \ref{Lclust} in the next section, neither the block size nor node degree grows in the asymptote.

\section{Local ( model + algorithm  + results )}\label{clustering}

This section investigates clustering, or community detection, in sparse and transitive networks.  Following the results of the last section,  sparse and transitive communities are small.  As such, this section is focused on finding small clusters of nodes.  In an attempt to strip away as many assumptions as possible from the Stochastic Blockmodel, this section 
\begin{enumerate}
\item proposes a ``localized" model with a small and transitive cluster embedded in a large and sparse graph (that could be chosen by an adversary), 
\item introduces a novel local clustering algorithm that explicitly leverages the graphs transitivity, and 
\item shows that this local algorithm will discover the cluster in the localized model with high probability.
\end{enumerate}
Similarly to the last section, the interaction between sparsity and transitivity provides for these results, enabling both the fast algorithm and the fixed block asymptotics.

\subsection{The local Stochastic Blockmodel} \label{localMod}
The ``local" Stochastic Blockmodel (defined below) presumes that a small set of nodes $S_*$ constitute  a single block and the model parameterizes how these nodes relate to each other and how they relate to the rest of the network.  
%In the theorem, the only assumption on the rest of the network is that the edges are sparse;  this part of the network could be deterministic (i.e. adversarial) or stochastic. 

%The local Stochastic Blockmodel defines how a fixed set of nodes $S_*$ relate to each other and how they relate to the rest of the network.  As for the rest of the network, the local model only assumes sparsity The section concludes with the local result in Theorem XX that implies with high probability (1) the local algorithm, initialized with any node in $S_*$, returns the set $S_*$  and (2) the global hierarchical clustering algorithm, cut at a specified level, will return the set $S_*$ as one of the clusters. 

\begin{definition} \label{localsbm}
 Suppose $A \in \{0,1\}^{(n+s) \times (n+s)}$ is an adjacency matrix on $n+s$ nodes.  If there is a set of nodes $S_*$ with $|S_*| = s$ and
\begin{enumerate}
\item $i,j \in S_*$ implies $P(A_{ij} =  1) \ge p_{in}$,
\item $i \in S_*$ and $j\in S_*^c$ implies $P(A_{ij} =1) \le p_{out}$,
\item the random variables $\{A_{ij}:  \forall i \in S_* \mbox{ and } \forall j \}$ are both mutually independent and also independent of the rest of the graph %and $j\ne j'$ implies $A_{ij}$ is independent of $A_{ij'}$,
\end{enumerate}
then $A$ follows the \textbf{local Stochastic Blockmodel} with parameters $S_*, p_{in}, p_{out}$.
\end{definition}

The only assumption that this definition makes about edges outside of $S_*$ (that is, $(i,j)$ with $i, j \not \in S_*$) is that they are independent of the edges that connect to at least one node in $S_*$.  So, the edges outside of $S_*$ could be chosen by an adversary, as long as the adversary  does not observe the rest of the graph.  The theorems below will add an additional assumption that the average degree (within $S_*^c$) must be not too large (i.e. it must be sparse).
%Of course, when $s$ is fixed and $n$ grows, 
%%th the most interesting asymptotic regime, $s$ remains fixed while $n$ grows.  So, 
%the induce subgraph on $S_*^c$ corresponds the majority of $A$, by far.

\subsection{Local clustering with transitivity} \label{localAlg}

%\textbf{Previous literature} 
%By minimizing the assumptions to include only ``local" assumptions for the edges connecting to $S_*$, the model echos a recent 
%Recently, a literature in computer science has focused on ``local" clustering algorithms that are initialized around a seed node and search for a tight community around this node.  
%Local clustering algorithms differ from standard clustering algorithms because 
A local  algorithm searches around a seed node for a tight community that includes this seed node.  Several papers have demonstrated the computational advantages of local algorithms for massive networks \citep{priebe2005scan,  spielman2008local}.   In addition to fast running times and small memory requirements, the local results are often more easily interpretable \citep{priebe2005scan} and yield what appear to be ``statistically regularized" results when compared to other, non-local techniques \citep{leskovec2009community,clauset2005finding,liao2009isorankn}.  \cite{andersen2006local, andersen2007detecting,andersen2009finding} have studied the running times and  given perturbation bounds showing that local algorithms can approximate the graph conductance.  With the exception of \cite{rukhin2012limiting} and \cite{scanPriebe}, the previous literature has not addressed the statistical properties of local graph algorithms under statistical models.

Given an adjacency matrix $A$ and a seed node $i$, this section defines an algorithm that finds a clusters around node $i$.  This algorithm has a single tuning parameter $cut$ that balances the size of the cluster with the tightness of the cluster.  Smaller values of $cut$ return looser clusters.  The algorithm initializes the cluster with the seed node $S = \{i\}$.  It then repeats the following step:  For every edge between a node in $S$ ($j \in S$) and a node not in $S$ ($\ell \in S^c$), add $\ell$ to $S$ if there are at least $cut$ nodes that connect to both $\ell$ and $j$ (this ensures that $(i,j)$ is contained in at least $cut$-many triangles).  Stop the algorithm if all edges across the boundary of $S$ are contained in fewer than $cut$-many triangles. % $(j, \ell)$ with $j \in S$ and  $\ell \in S^c$, $j$ and $\ell$ have fewer than $cut$ common neighbors.  
%Note that if $(j, \ell)$ is an edge, then the number of common neighbors between $j$ and $\ell$ is the number of triangles that contain the edge $(j, \ell)$.  As such, this algorithm measures the edgewise transitivity.  
%Later in the section, each triangle is weighted by the degrees of the nodes that compose the triangle.  

%In practice, it might help to normalize this quantity by a function of the degrees of node $j$ and $\ell$. 
%XXXXX
%
%Tai, can you try re-write the pseudo code for LocalTrans so that if someone were to exactly implement the algorithm, then it would be really fast?  In the current formulation, we say to compute $T$.  However, you only need to compute this matrix in the global version.  In the local version, you should only need one row at a time.  In the end, you should only need to compute $|S|$ rows of $T$.  Does that make sense? A nice thing about the current algorithm is that LocalTrans$(L,i,\tau)$ makes sense.  Ensure that this is still true in the new form. 
%
%I would suggest trying this package:  http://www.math-linux.com/latex-26/faq/latex-faq/How-to-write-algorithm-and  (make sure to copy the address from the .tex code).   This code forces you to write it out very cleanly... minimal english.  
%
%You might define an auxillary algorithm/function called tri (or whatever you want to call it).  $tri(i,\tau, A)$ returns the indices of the elements of $[AA]\cdot A$ that are greater than $\tau$.  
%XXXX
\begin{algorithm}
\caption{LocalTrans($A, i, cut$)}
\begin{algorithmic} 
\STATE 1. Initialize set $S$ to contain node $i$. 
%\WHILE{ $|S_*| \le size(A)$}
\STATE 2. For each edge $(i,\ell)$ on the boundary of $S$ ($i \in S$ and $\ell \notin S$) calculate $T_{i\ell}$:
\[T_{i\ell} = \sum_k A_{ik}A_{k\ell}.\]
\STATE 3. If there exists any edge(s) $(i,\ell)$ on the boundary of $S$ with $T_{i\ell} \ge cut$, then add the corresponding node(s) $\ell$ to $S$ and return to step 2. 
\STATE 4. Return $S$.

%Find node $s \in S_*$ and node $t \notin S_*$, such that 
%\[T_{st} = \max_{i \in S_*, j \notin S_*} T_{ij}.\]
%\IF{$T_{st} \ge cut$}
%\STATE add node $t$ to $S_*$
%\ELSE 
%\STATE return set $S_*$
%\ENDIF
%%\ENDWHILE

\end{algorithmic}
\end{algorithm}

%\vspace{.1in}

% \noindent
%\framebox[5.9 in][c]{  \begin{minipage}[c]{5.4in}
%\begin{large}
%\hspace{-.2in} \texttt{LocalTrans$(A,i,\tau)$:Local clustering with transitivity}
%
%\hspace{-.2in}  \texttt{Input: Adjacency matrix $A \in \{0,1\}^{n\times n}$, seed node $i$ and cutoff $\tau$.\\
%1. Get transitivity similarity matrix $T=(AA)\cdot A$, where $\cdot$ is element-wise multiplication;\\
%2. Initialize a set $S$ that only include node $i$;\\
%3. Iteratively adding nodes to $S$: at each iteration, find} 
%\[w=max_{j \in S, l \notin S} T(j, l),\quad l=argmax_{j \in S, l \notin S} T(j, l).\]
%\texttt{If $w \ge \tau$, add node $l$ to $S$, otherwise terminate.}
%
% \hspace{-.2 in} \texttt{Output: Local clusters $S$.} 
%\end{large}
% \end{minipage}
%}

\vspace{.2in}

Consider \texttt{LocalTrans}$(A,j,\tau)$ as a function that returns a set of nodes, then 
\[i \in \texttt{LocalTrans}(A,j,cut) \implies j \in \texttt{LocalTrans}(A,i,cut).\]  
Moreover, if $cut^+ > cut$ then, 
\[\texttt{LocalTrans}(A,i,cut^+) \ \subset \ \texttt{LocalTrans}(A,i,cut).\]  
This shows that the results of \texttt{LocalTrans}$(A,i, cut)$, for every node $i$ and every parameter $cut$, can be arranged into a dendogram.  %A fast  ``global" algorithm with an easy interpretation can compute this tree.
%Initialize $S = \{i\}$. \\
%For $i$ in $S$\\
%\ For $j \not \in S$ such that $(i,j) \in E$,\\
%\ \ If there exists at least $\tau$ nodes $k$ such that $(i,k) \in E$ and $(j,k) \in E$,
%\subsubsection{A global version}  %The solution to \texttt{LocalTrans} for all nodes $i$ and all tuning parameters $\tau$ can be arranged into a tree.  To see this, c
\texttt{LocalTrans} only finds one branch of the tree.  A simple and fast algorithm can find the entire tree.  

To compute the entire dendogram, apply single linkage hierarchical clustering\footnote{This is equivalent to finding the  maximum spanning tree} to the similarity matrix
\begin{equation}\label{defT}
T = (AA) \cdot A, \mbox{ where $\cdot$ is element-wise multiplication.}
\end{equation}
The computational bottleneck of this algorithm is computing $T$, which can be computed in $O(|E|^{3/2})$.  Techniques using fast matrix multiplication can slightly decrease this exponent \citep{alon1997finding}.
%Note that computing $T \in \mathcal{N}^{n \times n}$ only requires $[AA]_{ij}$ if $A_{ij} =1$.  As such, there are algorithms to compute $T$ in $O(|E|^{3/2})$.  Techniques using fast matrix multiplication can slightly decrease this exponent \citep{alon1997finding}.

When $A$ contains no self loops, $T_{ij}$ equals  the number of triangles that contain both nodes $i$ and $j$.  We propose \textit{single} linkage because it is the easiest to analyze and it yields good theoretical results.  However, in some simulation, \textit{average} linkage has performed better than single linkage.  One could also state a local algorithm in terms of average linkage.

\begin{algorithm}
\caption{GlobalTrans($A,\tau$)}
\begin{algorithmic} 
\STATE 1. Compute the similarity matrix $T = [AA]\cdot A$, where $\cdot$ is element-wise multiplication.
\STATE 2. Run single linkage hierarchical clustering on similarity matrix $T$, i.e. grow a maximum spanning tree.
\STATE 3. Cut the dendogram at level $\tau$, i.e. delete any edges in the spanning tree with weight smaller than $\tau$.
\STATE 4. Return the connected components.
\end{algorithmic}
\end{algorithm}

% \noindent
%\framebox[5.9 in][c]{  \begin{minipage}[l]{5.4in}
%\begin{large}
%\hspace{-.2in} \texttt{GlobalTrans$(A,\tau)$:Global clustering with transitivity}
%
%\hspace{-.2in}  \texttt{Input: Adjacency matrix $A \in \{0,1\}^{n\times n}$ and cutoff $\tau$.\\
%1. Get transitivity similarity matrix $T=(AA)\cdot A$, where $\cdot$ is element-wise multiplication;\\
%2. Grow a maximum spanning tree $T_{max}$ based on T;\\
%3. Delete the edges in $T_{max}$ that are smaller than $\tau$. This steps gives a new graph G'.}
%
% \hspace{-.2 in} \texttt{Output: each connected components of G'} 
%\end{large}
% \end{minipage}
%}
%
%\vspace{.2in}

\begin{prop}  Viewing \texttt{LocalTrans} as a function that returns a set of nodes and \texttt{GlobalTrans} as functions that returns a set of sets, $\texttt{LocalTrans}(A,i,\tau) \subset \texttt{GlobalTrans}(A, \tau)$. 
Moreover,
\[\bigcup_i \texttt{LocalTrans}(A,i,\tau) = \texttt{GlobalTrans}(A, \tau).\] 
\end{prop}
\begin{proof}
Nodes $i$ and $j$ are in the same cluster in both \texttt{LocalTrans}$(A,i,\tau)$ and \texttt{GlobalTrans}$(A,\tau)$ if and only if there exists a path from $i$ to $j$ such that every edge in the path is in at least $\tau$ triangles. 
\end{proof}

\subsection{Local Inference}

The next theorem shows that \texttt{LocalTrans} estimates the local block in the Local Stochastic Blockmodel with high probability.

%\begin{theorem}  Under the local Stochastic Blockmodel, if $p_{out} = o(1/\sqrt{\lambda n})$, then with probability greater than 
%\[lksdjfl;kasjdf\]
%the nodes in $S_*$ form a connected component in the weighted graph corresponding to $T$, as defined in Equation \ref{defT}.
%\end{theorem}
\begin{theorem} \label{Aclust} Under the local Stochastic Blockmodel (Definition \ref{localsbm}), if 
\[ \sum_{i,j \in S_*^c} A_{ij} \le n \lambda,\]
then
%if $p_{out} = O(n^{-1})$, then
\begin{enumerate}
\item \textbf{\textit{cut} = 1:} for all $ i \in S_*$, LocalTrans$(A,i,cut = 1) = S_*$  with probability greater than 
\[1 - \left(\frac{1}{2}s^2(1 - p_{in}^2)^{s - 2}     + O(p_{out}^2 n s (s+\lambda))\right).\]
%O(n^{-1})\right)\]

\item \textbf{\textit{cut} = 2:} for all $ i \in S_*$, LocalTrans$(A,i,cut = 2) = S_*$   with probability greater than 
\[1 - \left(s^3(1 - p_{in}^2)^{s - 3}     + O(p_{out}^3 n s(s + \lambda)^2)\right).\]
%O(n^{-2})\right)\]
\end{enumerate}
\end{theorem}

See the appendix for the proof of this theorem.

For the local Stochastic Blockmodel to create a transitive block with bounded expected degrees, it is necessary for $p_{in}$ to be bounded from below and for $s$ to be bounded from above.  Then, $p_{out} = O(1/n)$ is a sufficient condition for bounded expected degrees.  
%If, in addition, $p_{out} = O(
%
%For a large $p_{in}$, the nodes in $S_*$ create a transitive block.  The nodes in this block have expected degree bounded between $sp_{in}$ and $s + n p_{out}$.  So, when $p_{in}$ is non-vanishing and the expected degrees are bounded,  $s$ is also bounded.  When $s$ is bounded and $p_{out} = O(1/n)$, then the expected degrees are bounded.  
However, because of the inequality in bullet (2) of Definition \ref{localsbm}, $p_{out} = O(1/n)$ is not a necessary condition for sparsity.  As such, the restriction on $p_{out}$ is particularly relevant.  It is also fundamental to the bound in Theorem \ref{Aclust}.

The approximation terms $O(p_{out}^2 n s (s+\lambda))$ and $O(p_{out}^3 n s(s + \lambda)^2)$ bound the probability of a connection in $T$ across the boundary of $S_*$.  The strength of these terms come from the fact that a connection across the boundary requires $cut+1$ simultaneous edges across the boundary of $S_*$.  Figure \ref{fig:2tri} gives a graphical explanation for the ``+1".  As such, $p_{out}$ is raised to the \textit{cut +1} power.  When $s$ and $\lambda$ are fixed, then $p_{out} = o(n^{-1/2})$ and $p_{out}  = o(n^{-1/3})$ make the approximation term asymptotically negligible for the case $cut = 1$ and $cut =2$, respectively.  Both of these settings allow the nodes in $S_*$ have (potentially) large degrees.  If $p_{out} = O(n^{-1})$, then the approximation terms become $O(n^{-1})$ and $O(n^{-2})$ respectively.

\begin{figure}[htbp] %  figure placement: here, top, bottom, or page
   \centering
   \includegraphics[width=2in]{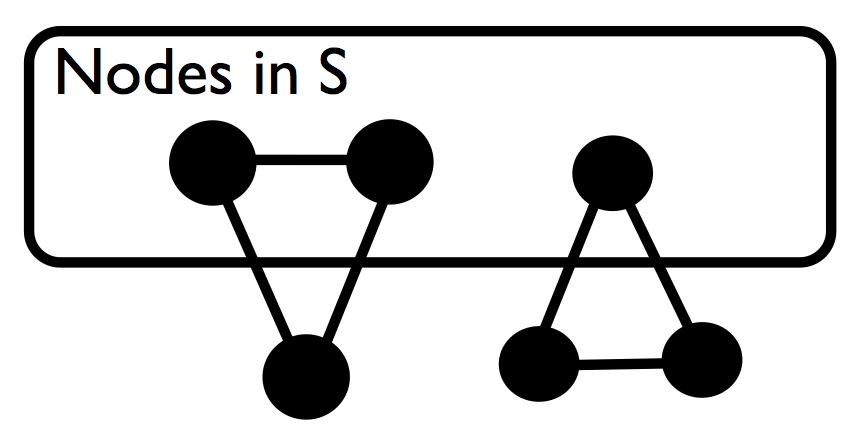} 
   \caption{This figure illustrates the two types of triangles that contain nodes in both $S_*$ and $S_*^c$.  To make \textit{one} triangle that crosses the boundary of $S$ requires \textit{two} edges to cross the boundary. 
%   In both types of triangles, there are two edges that cross the boundary.  Under the local Stochastic Blockmodel, each of these edges occur with probability $p_{out}$.  So, the probability of any specific triangle is less than or equal to $p_{out}^2$.  Because this term is squared, it allows for non-trivial $p_{out}$ probabilities in Theorem \ref{Aclust}.
}
   \label{fig:2tri}
\end{figure}

Theorem 2 and the algorithms \texttt{LocalTrans} and \texttt{GlobalTrans} all leverage the interaction between transitivity and sparsity, making the task of computing $S$ and estimating $S_*$ both algorithmically tractable and statistically feasible.  
%Theorem 2 makes a stronger conclusion ($p_{in} < 1$ and $|S|$ fixed) than is likely possible under the planted clique problem; the interaction between transitivity and sparsity provide the basic ingredients.

%Under the local sbm, the expected local clustering coefficient (see equation \ref{localcc}) for the nodes in $S_*$ can range between...  This suggests that this model allows for a wide stretch of transitivity.
 
\subsection{Preliminary Data Analysis}
This section applies \texttt{GlobalTrans} to an online social network from the website slashdot.org, demonstrating the shortcomings of the proposed algorithms with input $A$ and motivating the next section that uses the graph Laplacian as the input.  The slashdot network contains 77\,360 nodes with an average degree of roughly 12 \citep{leskovec2009community}.\footnote{This data can be downloaded at \texttt{http://snap.stanford.edu/data/soc-Slashdot0811.html}}  This network is particularly interesting because it has a smaller transitivity ratio $(.024)$ than the typical social network.  %This suggests that large parts of the graph have very little local structure.

\begin{figure}[htbp] %  figure placement: here, top, bottom, or page
   \centering
   \includegraphics[width=5in]{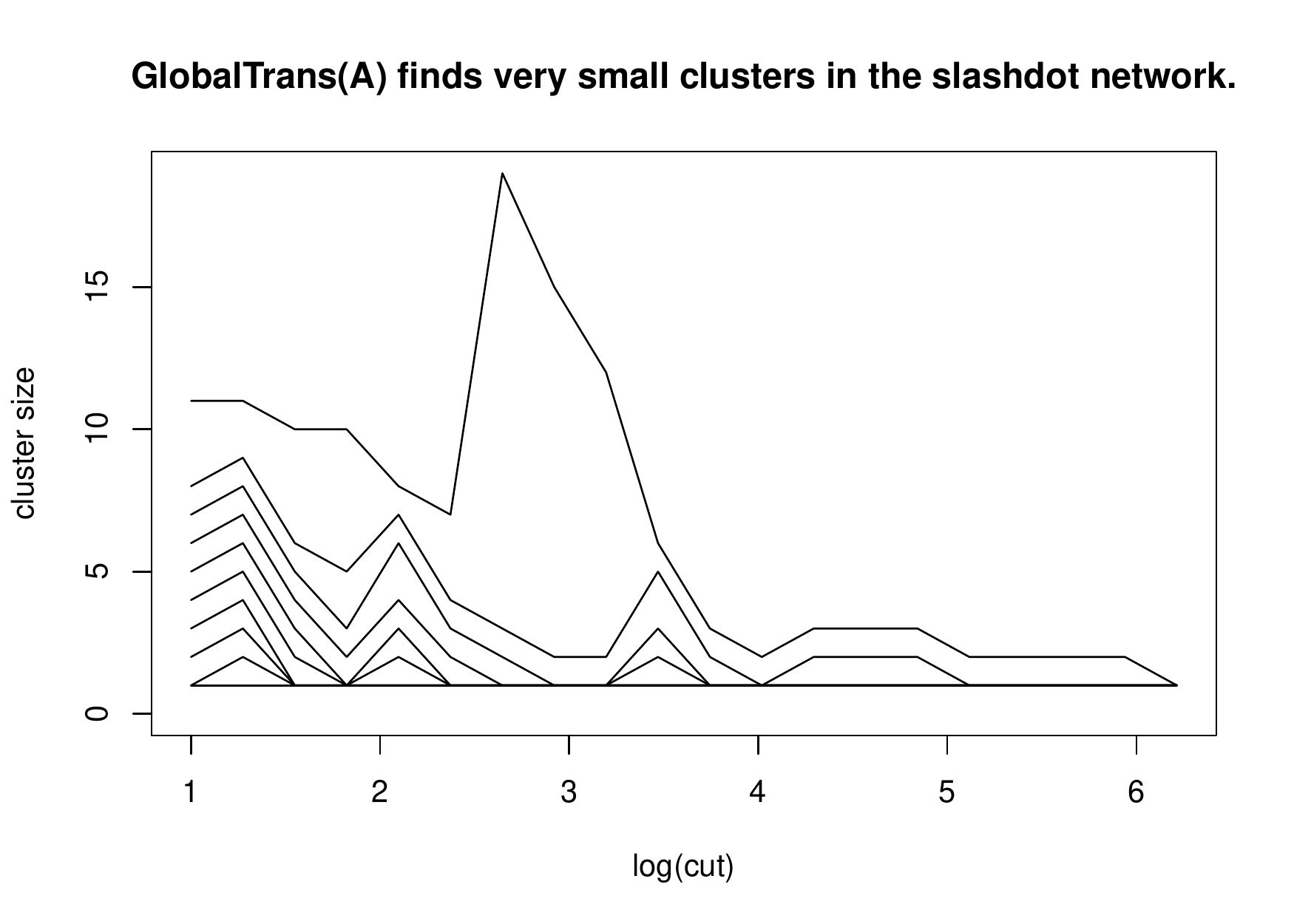}
   \caption{This plots the number of nodes in the largest ten clusters (ignoring a single giant cluster) found by \texttt{GlobalTrans}$(A, cut)$ in the slashdot social network.   These clusters are very small, and probably too small for many applications.  Moreover, there are not that many of them.  %The next section (1) provides a possible explanation for this,  (2) suggests a slight change to the algorithm, and (3) demonstrates that it finds more reasonably sized clusters.
   }
   \label{fig:A}
\end{figure}

Figure \ref{fig:A} plots the size of the ten largest clusters  returned by \texttt{GlobalTrans}$(A, cut)$ as a function of $cut$ (excluding the largest cluster that consists of the majority of the graph).  The values of $cut$ range from  3 to 500 and they are plotted on the $\log_{e}$ scale.  Over this range of $cut$, only two times does a cluster exceed ten nodes.  While we motivated the local techniques as searching for small clusters,  these clusters are perhaps \textit{too} small.  It suggests that there are no clusters that are adequately described by the local Stochastic Blockmodel.  

One potential reason for this failure is that under the Local Stochastic Blockmodel, the probability of a connection between a node in $S_*$ and a node in $S_*^c$ is uniformly bounded by some value, $p_{out}$. The slashdot social network, like many other empirical networks, has a long tailed degree distribution.  A more realistic model might allow the nodes in $S_*$ to be more highly connected to the high degree nodes in $S_*^c$.  The next subsection (1)  proposes a ``degree-corrected" local Stochastic Blockmodel, (2) proves that \texttt{LocalTrans} with a simple adjustment can estimate $S_*$ in the degree corrected model, and (3) demonstrates how this new version of the algorithm improves the results on the slashdot social network.

\section{The Degree-corrected Local Stochastic Blockmodel} \label{degClustering}
Inspired by \cite{karrer2011stochastic}, the degree-corrected model in Definition \ref{degLocal} makes the probability of a connection between a node $i \in S_*$ and a node $j \not \in S_*$ scale with the degree of node $j$ on the subgraph induced by $S_*^c$.  
\begin{equation}\label{dstardef}
d_j^* = \sum_{\ell \in S_*^c} A_{\ell j}
\end{equation}
  For the following definition to make sense, we presume that $d_j^*$ is fixed for all $j \in S_*^c$.

\begin{definition} \label{degLocal} Suppose $A \in \{0,1\}^{(n + s) \times (n  +s)}$ is an adjacency matrix and $S_*$ is a set of nodes with $|S_*| = s$.  For $j \in S_*^c$, define $d_j^*$ as in Equation \eqref{dstardef}.  If
\begin{enumerate}
\item $i \in S_*$ and $j\in S_*^c$ implies 
\[P(A_{ij} =1) \le \frac{d_j^*}{n},\]
\item $i,j \in S_*$ implies $P(A_{ij} =  1) \ge p_{in}$,
\item $\{A_{ij}: \forall j \mbox{ and }  \forall i \in S_* \}$ are mutually independent
%\item $ \sum_{i,j \in S_*^c} A_{ij} \le n \lambda,$
\end{enumerate}
then $A$ follows the \textbf{local degree-corrected Stochastic Blockmodel} with parameters $S_*, p_{in}$.
\end{definition}
The fundamental difference between the previous local model and this degree corrected version is the assumption that if  $i \in S_*$ and $j\in S_*^c$, then
\[P(A_{ij} =1) \le \frac{d_j^*}{n}.\]
In the previous model, $P(A_{ij} = 1) \le p_{out}$. This new condition can be interpreted as $P(A_{ij} =1) \le p_{out} \hspace{.04 in} d_j^* $ for $p_{out} =1/n$. In this degree-corrected model, the nodes in $S_*$  connect to more high degree nodes than they do under the previous local model.  

The degree corrected model creates two types of problems for \texttt{LocalTrans}$(A,i,\tau)$.  Because the high degree nodes in $S_*^c$ create many connections to the nodes in $S_*$, it is more likely to create triangles with two nodes in $S_*$.  Additionally, by definition, the high degree nodes outside of $S_*$ have several neighbors outside of $S_*$.  As such, it is more likely to create triangles with one node in $S_*$ and two nodes outside of $S_*$.  In essence, the high degree nodes create several triangles in the graph, washing out the clusters that \texttt{LocalTrans}$(A,i,\tau)$ can detect.  To confront this difficulty, it is necessary to down weight the triangles that contain high degree nodes.

%In the standard degree-corrected Stochastic Blockmodel, the probability of a connection between $i$ and $j$ is  
%\[P\left(A_{ij} =1 |(\xi_i, \theta_i),  (\xi_j,\theta_j)\right)  = \frac{\Theta_{\xi_i, \xi_j} \sqrt{\theta_i \theta_j}}{n},\] 
%where $\xi_i$ and $\xi_j$ are the block labels, $\Theta$ is the standard $K\times K$ matrix of probabilities, and $\theta_i$ is a latent parameter that allows for heterogeneous degrees.  Under certain parameterizations, it is the expected degree or proportional to the expected degree.  Where the standard degree corrected model presumes that the (expected) degree enters as $\sqrt{E(d_j)}$
%
%A fundamental difference between the local degree corrected model and the standard degree corrected model is that the local model presumes that connections between $i \in S_*$ and $j \in S_*^c$ scale \textit{linearly} with the degree of $j$ (within $S_*^c$)

\subsubsection{The graph Laplacian}

%The spectral clustering algorithm, which down weights the high degree nodes with the normalized graph Laplacian , 
%
%To confront the heterogeneous degrees by passing the normalized graph Laplacian (instead of the adjacency matrix) to the \texttt{GlobalTrans} and \texttt{LocalTrans} algorithms.

Similarly to the adjacency matrix, the normalized graph Laplacian represents the graph as a matrix.  In both spectral graph theory and in spectral clustering, the graph Laplacian offers several advantages over the adjacency matrix \citep{chung1997spectral, von2007tutorial}.  The spectral clustering algorithm uses the eigenvectors of the normalized graph Laplacian,  not the adjacency matrix, because the normalized Laplacian is robust to high degree nodes \citep{von2007tutorial}.

For adjacency matrix $A$,  define the diagonal matrix $D$ and the normalized graph Laplacian $L$, both elements of $R^{n \times n}$, in the following way
\begin{equation}
	\label{Ldef}
    \begin{array}{rll}
D_{ii} & = & d(i)  \\
L_{ij} &=& [{D}^{-1/2}A{D}^{-1/2}]_{ij} = \frac{A_{ij}}{\sqrt{D_{ii}D_{jj}}} .
\end{array}
\end{equation}
Some readers may be more familiar defining $L$ as $I - D^{-1/2}WD^{-1/2}$. For our purposes, it is necessary to drop the $I-$.    

The last section utilized the matrix $T = [AA]\cdot A$ to find the triangles in the graph. To confront the degree corrected model, the next theorem uses $[LL]\cdot L$ instead.  The interpretation of this matrix is similar to $T$.  It differs because it down weights the contribution of each triangle by the inverse product of the node degrees.  For example,  where a triangle between nodes $i,j,k$ would add 1 to element $T_{ij}$, it would add $(d(i)d(j)d(k))^{-1}$ to the $i,j$th element of $[LL]\cdot L$.  

Some versions of spectral clustering use the random walk graph Laplacian, an alternative form of the normalized graph Laplacian.
\[L_{RW} = D^{-1}A\]
While the algorithmic results from spectral clustering can be depend on the choice of graph Laplacian,
%our algorithm is independent of the choice of normalization; the analogue of the matrix $T$ constructed with $L$ is equivalent to the construction with $L_{RW}$.  As such, 
\texttt{LocalTrans} returns exactly the same results with $L$ as it does with $L_{RW}$.    To see this, first imagine that if the graph is directed, then $A$ is asymmetric, and for $T$ to correspond to \textit{directed} cycles of length three, it is necessary to take the transpose of the final $A$,  that is $[AA]\cdot A^T$.  Since $L_{RW}$ is asymmetric, it is reasonable to use the additional transpose from the directed formulation.  It is easy to show that 
\begin{equation} \label{equivalentLap}
[LL]\cdot L = [L_{RW}L_{RW}]\cdot L_{RW}^T. 
\end{equation}

\cite{chaudhuri2012spectral} and \cite{chen2012fitting} have recently proposed a ``regularized" graph Laplacian.  \cite{chaudhuri2012spectral} propose replacing $D$ with $D_\tau = D + \tau I$, where $\tau >0$ is a regularization constant. They show that a spectral algorithm with
\[L_\tau = D_\tau^{-1/2}AD_\tau^{-1/2} \ \] 
has superior performance on sparse graphs.  Similarly, it will help to use $L_\tau$ with \texttt{LocalTrans}.  (Note that the equivalence in Equation \eqref{equivalentLap} still holds with the regularized versions of the Laplacians.)

The next theorem shows that under the local degree-corrected model --- with the regularized graph Laplacian, a specified choice of tuning parameter $cut$, and $i \in S_*$ ---
%$S_* \subset$ \texttt{GlobalTrans}$(L_\tau, cut)$ and
the estimate \texttt{LocalTrans}$(L_\tau, i, cut) = S_*$  with high probability.  Importantly, using $L_\tau$ instead of $A$ allows for reasonable results under the degree-corrected model.

\begin{theorem} \label{Lclust}
Let $A$ come from the local degree-corrected Stochastic Blockmodel.  Define $\lambda$ such that 
\begin{equation} \label{lambdaBound}
 \sum_{i,j \in S_*^c} A_{ij} \le n \lambda.
\end{equation}
Set $cut = [2(s-1)p_{in} + 2\lambda+\tau]^{-3}$.
If 
\[n \ge 3\left(2(s-1)p_{in} + 2\lambda+\tau\right)^{3/\epsilon}\tau^{-1/\epsilon},\]
 and $s \ge 3$, then for any $ i \in S_*$,
\[\texttt{LocalTrans}(L_\tau, i, cut) = S_*\]
 with probability at least 
\[1 - \left(\nicefrac{1}{2} \ s^2(1 - p_{in}^2)^{s - 2}+ s \exp\left(- \nicefrac{1}{4} \ sp_{in} + \lambda\right) + O(n^{3\epsilon - 1})\right).\]

%XXXX Tai, should the O() term be inside the $\exp$? I don't think soXXXX
%where  $0< \epsilon <\frac{1}{3}$
%, if 
%\[n \ge (\frac{3\{2(s-1)p_{in} + 2\lambda+\tau\}^3}{\tau})^{1/\epsilon},\]
%,  with
%\[cut = \frac{1}{\{2(s-1)p_{in} + 2\lambda+\tau\}^3},\]
% correctly identifies $S^*$ with probability at least 
\end{theorem}

A proof of Theorem \ref{Lclust} can be found in the Appendix.  

Because simple summary statistics (of sparsity and transitivity) on empirical networks contradict the types of models studied in the literature,  Theorem \ref{Lclust} tries to minimizes the assumptions on the ``global" structure of the graph.  It only assumes that the graph outside of $S_*$, i.e. the induced subgraph on $S_*^c$, is sparse.  There are no other assumptions on this part of the graph.  

This result is asymptotic in $n = |S_*^c|$, with  $S_*$ fixed and containing nodes with bounded expected degree;  the assumption in Equation \ref{lambdaBound} and the definition of $d_j^*$ imply that the nodes in $S_*$ have expected degree less than $s + \lambda$.

\subsection{Preliminary Data Analysis}
\begin{figure}[htbp] %  figure placement: here, top, bottom, or page
   \centering
   \includegraphics[width=5in]{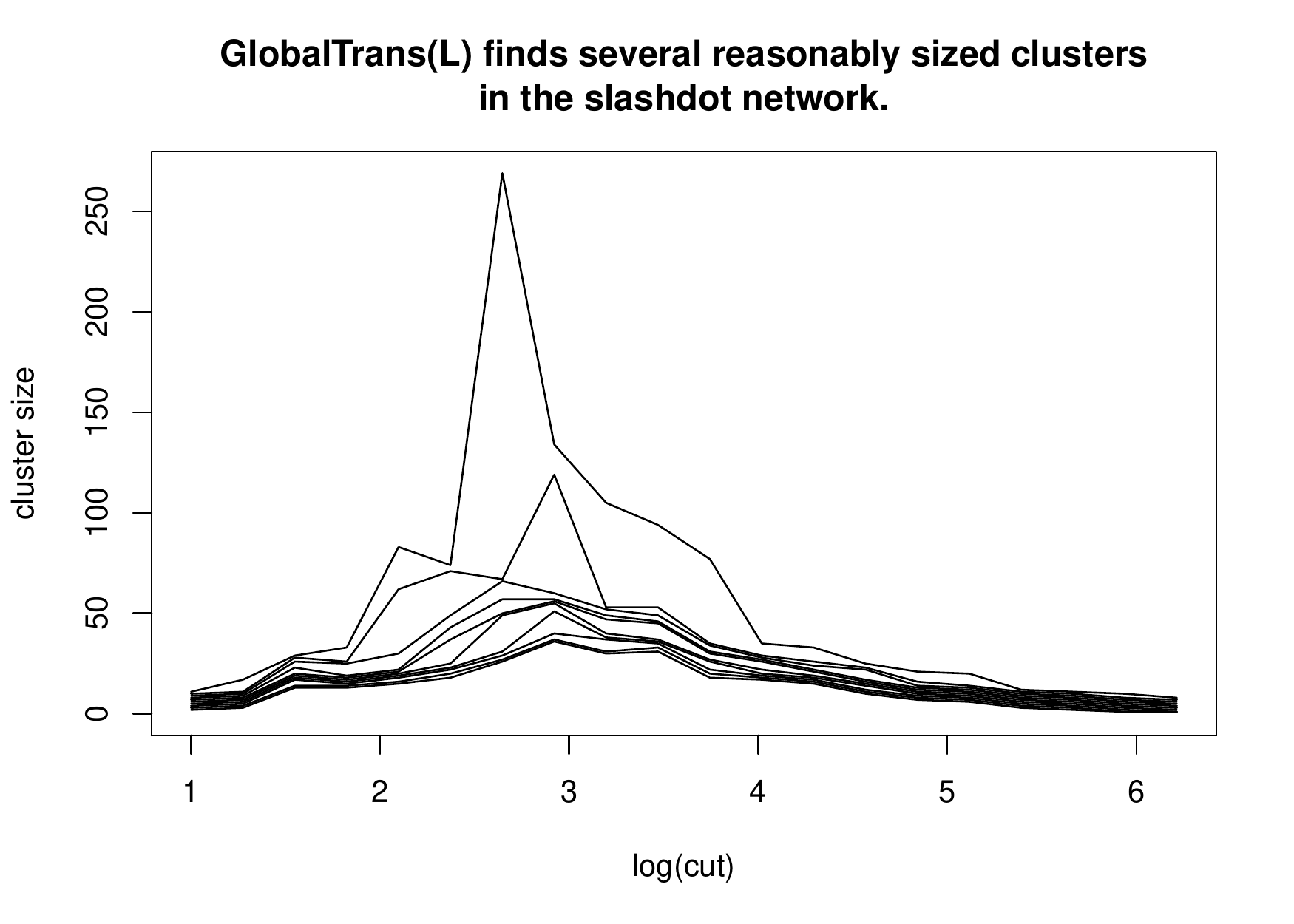}
   \caption{A plot of the number of nodes in the largest ten clusters (ignoring one very large cluster) found by \texttt{GlobalTrans}$(L_{\tau}, cut)$ in the slashdot social network.   \texttt{GlobalTrans}  with $L_\tau$ instead of $A$ finds much larger clusters. }
   \label{fig:L}
\end{figure}

\begin{figure}[htbp] %  figure placement: here, top, bottom, or page
   \centering
   \includegraphics[width=6in]{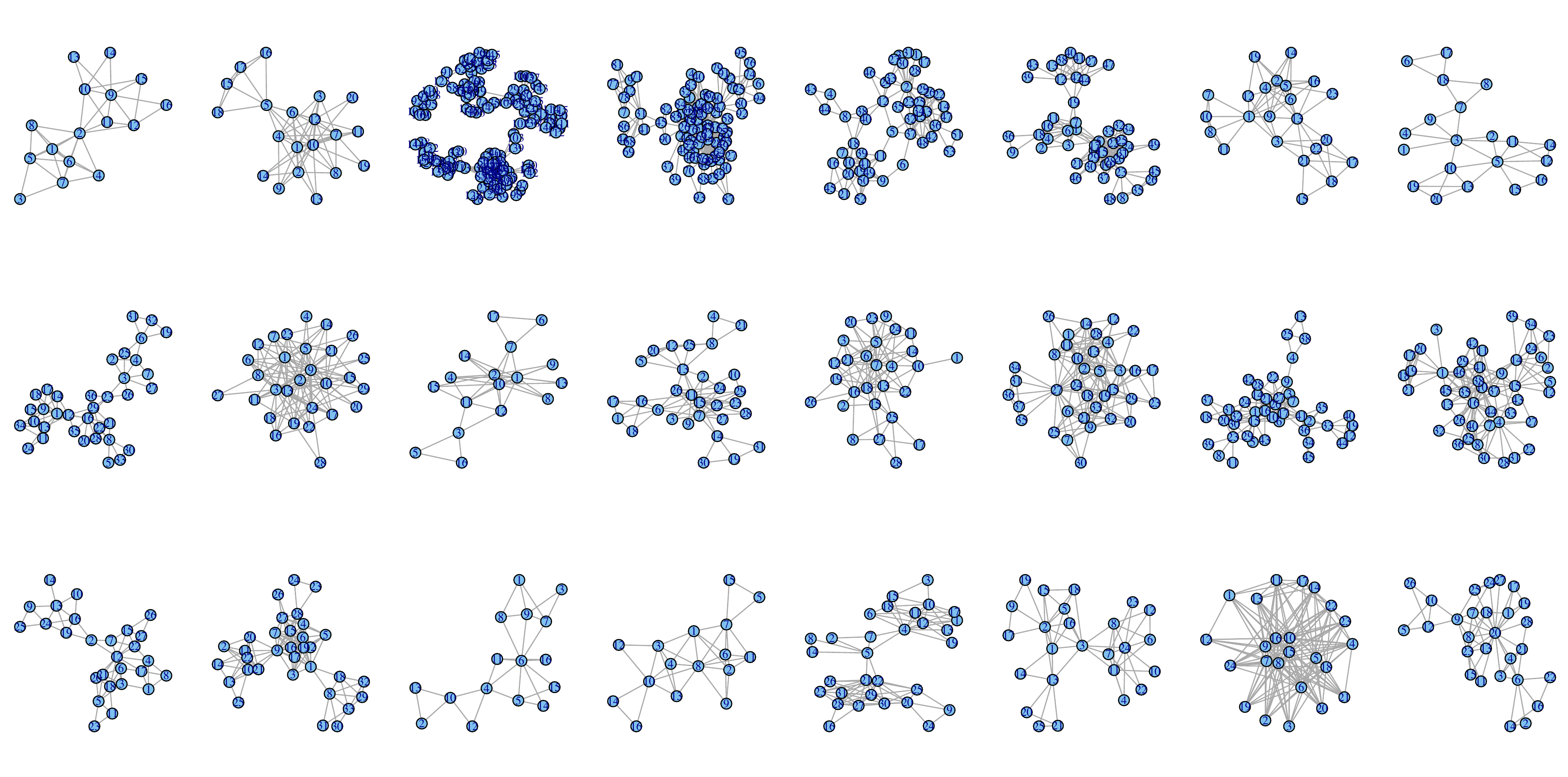} 
   \caption{   Twenty-four small clusters from the slashdot data set.   Because \texttt{GlobalTrans} discovers small clusters, one can easily plot and visualize the clusters with a standard graph visualization tool \citep{igraph}.  The point of this figure is to show the variability in cluster structures;  some are tight, clique-like clusters; others are small lattice-like clusters; others are ``stringy" collections of three or four tight clusters.  This highlights the ease of visualizing the results of local clustering.}
   \label{fig:SmallClusters}
\end{figure}

\begin{figure}[htbp] %  figure placement: here, top, bottom, or page
   \centering
   \includegraphics[width=6in]{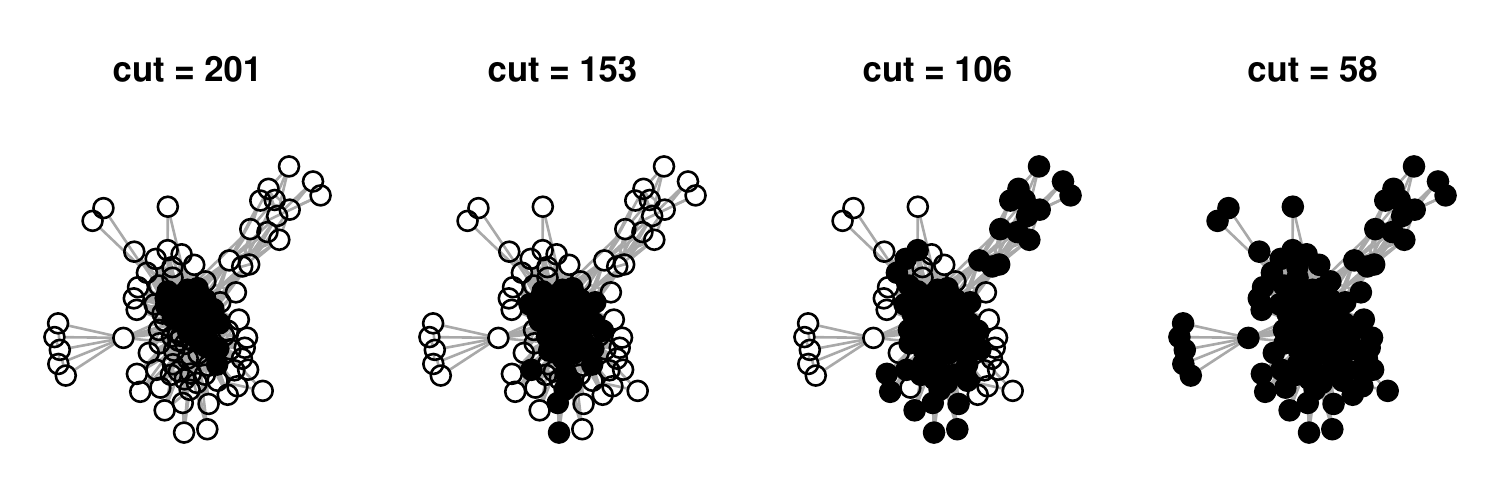}
   \caption{Starting from a seed node, this figure demonstrates how \texttt{LocalTrans}$(L_{\tau = 12},i,  cut)$ grows as $cut$ decreases.    In each panel, the graph is drawn for the smallest value of $cut$, and the solid nodes correspond to the nodes returned by \texttt{LocalTrans}$(L_{\tau = 12},i,  cut)$, where the value of $cut$ is given above the graph in the units $10^{-6}$. Moving from left to right, the clusters grow larger, and the additional nodes start to extend to the periphery of the visualization.    }
   \label{fig:Tree}
\end{figure}

Recall that Figure \ref{fig:A} illustrates how \texttt{GlobalTrans}$(A, cut)$ fails to find any clusters larger than twenty nodes in the slashdot social network. Figure \ref{fig:L} shows that using $L_\tau$ instead of $A$ corrects for the problems observed in Figure \ref{fig:A}. It plots the size of the largest ten clusters in the slashdot social network found by \texttt{GlobalTrans}$(L_{12}, cut)$ for values of $cut$ between $3*10^{-6}$ and $500*10^{-6}$.  It finds several clusters that exceed twenty nodes.  In this analysis, and all other analyses using $L_\tau$, the regularization constant $\tau$ is set equal to the average node degree (as suggested in \cite{chaudhuri2012spectral}).  In this case,  $\tau \approx 12$. 
% XXXX Tai, I'm sorry, I forgot again... does \cite{chaudhuri2012spectral} suggest using the average degree or was that your result?  If they don't say anything in that regard, we can say that it is from our unpublished work. They do suggest this choice XXX

Figure \ref{fig:SmallClusters} shows some of the clusters from the slashdot social network.  Specifically, it plots twenty-four of the induced subgraphs from \texttt{GlobalTrans}$(L_{\tau = 12}, cut = 32*10^{-6})$.   Because the clusters are not so large, the sub-graphs are easily visualized and it is easy to see how these clusters have several different structures.  Some are nearly planar; others appear as densely connected, ``clique-like" sub-graphs; other clusters are a collection of several smaller clusters, weakly strung together.   Figure \ref{fig:epinionClusters} in the introduction gives a similar plot for the epinions social network. These visualizations were created using the graph visualization tool in the igraph package in R \citep{igraph}. 

Figure \ref{fig:Tree} illustrates how \texttt{LocalTrans}$(L_{\tau = 12},i,  cut)$ changes as a function of $cut$ for a certain node in the epinions social network.    Each of the four panels displays the network for $cut = 58*10^{-6}$.  In each of the four panels, solid nodes are the nodes that are included in \texttt{LocalTrans}$(L_{\tau = 12},i,  cut)$ for four different values of $cut$.  This seed node was selected because the local cluster is slowly growing as $cut$ decreases and you can see in this in Figure \ref{fig:Tree}.\footnote{In particular, it was chosen as the ``slowest growing" from a randomly chosen set of 200 nodes.}  The left most panel displays the results for the largest value of $cut$.  This returns the smallest cluster and not surprisingly, the igraph package plots these nodes in the center of the larger graph.  Moving to the right, the clusters grow larger and the additional nodes start to extend to the periphery of the visualization.  While the clusters for this node grow slowly, for many other nodes, the transitions are abrupt.  For example, the nodes that join the cluster in the last panel in Figure \ref{fig:Tree} jump from cluster sizes of one or two into this bigger cluster.  Then, decreasing $cut$ a little bit more, this cluster becomes part of a giant component.

\section{Discussion}
The tension between transitivity and sparsity in networks that implies that there are local regions of the graph that are dense and transitive.  This leads to the blessing of dimensionality, which says that edges (in sparse and transitive graphs) become asymptotically more informative.  For example, under the exchangeable model, if the model is sparse and transitive, then the conditional density of the latent variables $\xi_i, \xi_j$, given $A_{ij} = 1$, is asymptotically unbounded, concentrating on the values of   $\xi_i, \xi_j$ that are consistent with the local structure in the model.
This has important implications for statistical models, methods, and estimation theory.  

%\ref{clustering} demonstrates how  
% The interaction between transitivity and sparsity has important implications for statistical models, theory, and methodology.  
% The previous literature on network inference has not studied models that are both sparse and transitive.   We have addressed this gap in the literature by proposing and studying the local (degree corrected) Stochastic Blockmodel.
%%In order to produce sparse models, previous research using the exchangeable random graph model (and the Stochastic Blockmodel) has often studied an asymptotic setting where $p_{\max}$ converges to zero.  Unfortunately, this has the bi-product of removing transitivity.  Interestingly, physics literature on networks has also struggled to model transitivity in sparse networks \citep{newman2009random}.  
%To create such a model, it is necessary for the model to retain some local structure.   For example, Theorem \ref{sbmtrans} shows that a Stochastic Blockmodel with fixed block size can be both asymptotically sparse and asymptotically transitive.  
%%In this model, the number of blocks grows proportionally to $n$.  So, each block is very small. In essence, this creates local structure.  

In sparse and non-transitive Stochastic Blockmodels, the block structure is \textit{not} revealed in the local structure of the network.  Rather, the blocks are revealed by comparing the edge density of various partitions. However, under transitive models, the local structure of the network can reveal the block structure.  As such, these blocks can be estimated by fast local algorithms.  
% in the low dim sbm are not local features of the net.  by making $K$ grow proportionally with $N$, the hdsbm can model sparse and trans networks.  the trans makes the blocks a local feature, allowing fast local algorithms to discover the structure. in previous research, lesk found that the local alg often return clusters of higher conductance than those found by other techniques.  However, in their words, local alg "found reg clusters, lower diam" 
%
%
%transitive clusters are a Local feature of the network.  As such, ``global" algorithms 
Theorems \ref{Aclust} and \ref{Lclust} show that \texttt{LocalTrans}  performs well under a \textit{local} Stochastic Blockmodel that makes minimal assumptions on the nodes outside of the true cluster; this is the first statistical result to demonstrate how local clustering algorithms can be robust to vast regions of the graph.
This paper studies small clusters because  (1) they  can create sparse \textit{and} transitive Stochastic Blockmodels, (2) they are relatively easy to find, both computationally and statistically, and (3) they are easy to plot and visualize. 
In future research, we will study how these ideas can be used to find large partitions in networks.  Sparse and transitive models do not preclude large partitions, as long as  some type of local structure exists within each partition.  It is not yet clear how global algorithms like spectral clustering might leverage this transitive structure in a stochastic model; this is one area for future research.

%\section*{Acknowledgements}

\bibliographystyle{plainnat}
\bibliography{citeTrans.bib}

\begin{appendix}
\section{Proofs for Section \ref{sec:trans}}

\subsection{Proof of Theorem \ref{sampleTrans}}

\begin{proof}
Recall that the transitivity ratio of $A$ is
\[trans(A) = \frac{ \mbox{number of closed triplets in $A$}}{\mbox{number of connected triples of vertices in $A$}}.\]
Both the numerator and the denominator of the transitivity ratio have other formulations that suggest how they can be  computed.  
%For example, let $\lambda_1 \ge \dots \ge \lambda_n$ denote the eigenvalues of $A$, then
\begin{eqnarray*}
\mbox{number of closed triplets in $A$} &=&6 \times \mbox{Number of triangles in $A$} \\&=&  \mbox{trace}(AAA) \\ % = \tfrac{1}{6} \sum_j \lambda_j^3\\
\\
\mbox{number of connected triples of vertices in $A$} &=&2 \times \mbox{Number of 2-stars in $A$}\\&=& 2 \sum_j {d_j \choose 2} \\& =& \sum_j d_j^2 - d_j.
\end{eqnarray*}

For ease of notation, define $X_n = trace(AAA)$ and $Y_n = \sum_i d_i^2  - d_i$.  So, $trans(A) = X_n/  Y_n$.  To show that transitivity converges to zero, use
\[P\left(\frac{X_n}{Y_n} > \epsilon\right) \le \frac{E(X_n/Y_n)}{\epsilon}\]
and the following Lemma.

\begin{lemma} \label{yflemma}
\textit{If $\lambda_n = o(n)$, then there exists a sequence $f_n$ such that $E(X_n) = o(f_n)$ and}
\[P(Y_n \ge f_n) \rightarrow 1.\]
\end{lemma}

Using Lemma \ref{yflemma} and fact that $X_n/Y_n \le 1$ a.s.,
\begin{eqnarray*}
E\frac{X_n}{Y_n} & \le &  E\left(\frac{X_n}{f_n}1\{Y_n > f_n\} + 1\{Y_n < f_n\}\right)\\
& \le &  \frac{E(X_n)}{f_n} + P(Y_n < f_n)\\
& \rightarrow & 0.
\end{eqnarray*}
Now, to prove Lemma \ref{yflemma}.  For ease of notation, define $d = \sum_i d_i$.  From Bickel, Levina, Chen, define $\hat \rho = \frac{d}{2n(n-1)}$. They show that $\hat \rho/\rho_n  \stackrel{P}{\rightarrow}  1$, where $\rho_n = P(A_{12} = 1)$.  So, this converges to zero:
%\begin{eqnarray*}
\[P\left(\frac{d}{2n(n-1)\rho_n} < 1/2 \right) =P\left(d < n(n-1)\rho_n\right)\]
%\end{eqnarray*} 
Define $M_n = n(n-1)\rho_n$. Then, $P\left(    d > M_n \right) \rightarrow 1$.  Define $f_n = M_n^2/n - M_n$. 
Notice that 
\[min_{\sum_i d_i = m} \sum_i d_i^2 - d_i \ge n ((m/n)^2 - m/n) = m^2/n - m.\]
Putting these pieces together,
\begin{eqnarray*}
P(Y_n \ge f_n) &=& \int_\Omega 1\{Y_n \ge f_n\} dP\\
&\ge& \int_{d >M_n} 1\{\sum_i d_i^2 - d \ge f_n\} dP\\
&=& \sum_{m > M_n} \int_{d=m} 1\{\sum_i d_i^2 - d \ge f_n\} dP\\
&\ge& \sum_{m > M_n} \int_{d =m} 1\{m^2/n - m \ge f_n\} dP\\
&=& \sum_{m > M_n} \int_{d =m}  dP\\
&=&P(d > M_n) \rightarrow 1.
%&=& \int_{\sum_{i<j}A_{ij} >M} 1\{\sum_i D_i^2 >t + \sum_i D_i \} dP\\
% & \le &  E\left(\frac{X_n}{f_n}1\{Y_n > f_n\} + 1\{Y_n < f_n\}\right)\\
%& \le &  \frac{p_{ct} E(Y_n)}{f_n} + P(Y_n < f_n).
\end{eqnarray*}

The last piece is to show that $E(X_n) = o(f_n)$.  From the definition of $f_n$ and the fact that $\rho_n = \lambda_n/n$,
\begin{eqnarray*}
f_n &=& \frac{\left( n(n-1)\rho_n\right)^2}{n} -  n(n-1)\rho_n\\ 
&=& n(n-1)^2(\lambda_n/n)^2 -  n(n-1)\lambda_n/n\\ 
&=&  \lambda_n (n-1)\left(\frac{n-1}{n} \lambda_n - 1 \right)\\
&\rightarrow&  \lambda_n^2 n%\\
%&=&  O( \rho_n^2 n^3).
\end{eqnarray*}

Define 
\[S_{12} = \sum_i A_{i1} A_{2i} \]
as the number of two stars with nodes 1 and 2 as end points. Then, under the assumption that
\begin{equation*} \label{pbound}
p_{\max} = o\left(\frac{P(A_{13} = 1)} {P(A_{13} = 1|A_{23} = 1)}\right),
\end{equation*}
it follows that 
\begin{eqnarray*}
E(X_n) &=& p_{ct} n(n-1) E(S_{12}) \\
&\le& p_{\max} n^2 (n E(A_{13}A_{23})) \\
&=& p_{\max} n^3 E(A_{13}|A_{23} = 1) E(A_{23} = 1) \\
&=& o(\rho_n^2 n^3) \\
&=& o(f_n).
\end{eqnarray*}

%Where the last line holds when $p_{\max} = o(\rho_n^{2/3})$.

%Alternatively,
%\begin{eqnarray*}
%E(X_n) &\le& p_{\max} n^3 E(A_{13}A_{23}) \\
%&\le& p_{\max} n^3 E(A_{13}|A_{23} = 1) \rho_n \\
%&\le& p_{\max} n^3 p_{\max} \rho_n \\
%&=& o(\rho_n^2 n^3),
%\end{eqnarray*}
%when $p_{\max} = o(\rho^{1/2})$.  I think this is a nice assumption.

%Another thought...
%\begin{eqnarray*}
%E(X_n) &\le& p_{\max} n^3 E(A_{13}|A_{23} = 1) \rho_n \\
%&=& o(\rho_n^2 n^3) = o(f_n),
%\end{eqnarray*}
%when 
%\begin{equation} \label{pbound}
%p_{\max} = o\left(\frac{P(A_{13}} {E(A_{13}|A_{23} = 1)}\right).
%\end{equation}
\end{proof}

\subsection{Proof of Theorem~\ref {sbmtrans}:} 
\begin{proof}
Let $r = c_0/n$ and let $p$ be fixed. 

{\bf Number of triangles:} Let $\Delta_n$ denote the number of triangles. Notice that there are three types of triangles: (1) let $\Delta^i$ denote the number of triangles with all nodes in block $i$; (2) let $\Delta_{21}$ denote the number of triangles with 2 nodes in the same block and one node in a separate block; (3) let $\Delta_{111}$ denote the number of triangles with nodes in three separate blocks.  
\[E(\Delta_{21}) = K(K-1){s \choose 2}spr^2  = K(K-1){s \choose 2}spc_0^2 / (s^2K^2) \le p(s-1)/2,\]
\[E(\Delta_{111}) = {K \choose 3} s^3 r^3 \le K^3 s^3 c_0^3 / (6 s^3 K^3) = c_0^3 / 6.\]
By the Markov inequality, $\Delta_{21}/K \stackrel{P}{\rightarrow} 0, \Delta_{111}/K \stackrel{P}{\rightarrow} 0$. Finally, $\Delta^i$ are iid.   So, by LLN, their average converges in probability to their expectation.  Putting these pieces together with Slutsky's theorem,  the number of triangles over $K$ is,
\[\frac{1}{K}\Delta_n \stackrel{P}{\rightarrow} E(\Delta^i) = {s \choose 3}p^3.\]
{\bf Number of two stars:} Let $S_n$ be the number of two-stars. Define the events $B = \{|S_n - E(S_n)| > t\}$ and $A = \{\mbox{Maximum Degree} \le M\}$
\[P(B) = P(B A) + P(B A^c) \le P(B A) + P(A^c)\]
Apply the bounded difference inequality within the set $BA$. Define $A_i \in \{0,1\}^{n-i}$ for $i = 1, \dots, n-1$ as the $i$th row of the upper triangle of the adjacency matrix $A$. To bound the bounded difference constant, first notice that $c_1\ge c_i$ for all $i$.  Moreover, we have
\[c_1 \le 3 M^2.\]
This is because node $1$ belongs to at most $3{M \choose 2}$ triplets. By changing the edges of node $1$, $S_n$ can increase or decrease by at most $3{M \choose 2}$.
By the bounded difference inequality, 
\[P(B A) \le 2 \exp\left(-\frac{t^2}{n9M^4}\right).\]
Choose $M = \log n$ and $t = K \epsilon$ for any $\epsilon > 0$, we have $P(B A) \rightarrow 0$.
More over, by concentration inequality, 
\begin{align*}
P(A^c) = P(\cup_i\{d_i \ge M\}) \le n P(d_1 \ge M) \rightarrow 0
\end{align*}
Therefore, we have 
 \[S_n/K \stackrel{P}{\rightarrow} E(S_n)/K.\]
Notice that $E(S_n)/K$ is equal to the expected number of two-stars whose center is in the first block.  So,
\begin{eqnarray*}
E(S_n)/K &=& s\left( {s-1 \choose 2} p^2 + (s-1)(n-s)pr + { n - s \choose 2 } r^2\right)\\
& \rightarrow& s\left( {s-1 \choose 2} p^2 + (s-1)pc_0 + c_0^2/2\right) 
\end{eqnarray*}
Finally, 
\begin{align*}
TranRatio(A) = \frac{3 \times \mbox{number of triangles}}{\mbox{number of 2-stars}} \stackrel{P}\longrightarrow \frac{3E\Delta_n}{ES_n}.
\end{align*}

\end{proof}

\section{Proofs for Section \ref{clustering} and \ref{degClustering}}

\subsection{Proof of Theorem~\ref{Aclust}:} 

\begin{proof} Define the following events 
%XXX Tai, we should reserve $i$ for an index over nodes. What other symbol could we use?  Nothing pops into my head immediately... $i,j,k,\ell$ are taken. $\tau$ is taken.  $c$ would be cute, but we don't want it to be confused with complement... Maybe one of these:  $\kappa, \delta,\gamma,\alpha$??
\begin{eqnarray*}
B_\alpha &=& \{\mbox{$S_*$ and $S_*^c$ are separated with cutting level $\alpha$}\},\\
C_\alpha &=& \{\mbox{$S_*$ is clustered within one block with cutting level $\alpha$}\},\\
D_\alpha &=& \{\mbox{Every pair of nodes in $S_*$ have at least $\alpha$ common neighbor}.\}
\end{eqnarray*}
If both events $B_\alpha$ and $C_\alpha$ are satisfied, then for any $i \in S_*$, LocalTrans$(A, i, \alpha)$ recovers block $S_*$ correctly.
Events $D_\alpha$ implies that $S_*$ is clustered within one block with cutting level $\alpha$, that is $D_\alpha \in C_\alpha$. To see this, assume the contrary, then there exists a partition $S_* = S_1 \cup S_2$, such that for any $j \in S_1, k \in S_2, T(j,k) < \alpha$. However, $D_\alpha$ implies that $S_*$ is connected, hence there exists $u \in S_1, v \in S_2$, such that $A(u,v) = 1$. Moreover, $D_\alpha$ also implies that $u,v$ have at least $\alpha$ common neighbors. Hence, $T(u,v) \ge \alpha$. This is a contradiction.

The following lemma leads to the desired results.
\begin{lemma} \label{lem1} Under the conditions above,
\begin{eqnarray*}
P(B_1^c)  &=& O(p_{out}^2ns(s + \lambda)),\\
P(B_2^c)  &=& O(p_{out}^3ns(s+\lambda)^2),\\
P(D_\alpha^c) &\le& \frac{1}{2}s^2\sum_{k = 0}^{\alpha-1} {s-2 \choose k} {(1 - p_{in}^2)}^{s-2-k}.
\end{eqnarray*}
\end{lemma}

By Lemma~\ref{lem1}, we have
\begin{eqnarray*}
P(\mbox{correctly clustering $S_*$ with cutting level 1}) &\ge & P(B_1 \cap C_1) \\
&=& 1 - P(B_1^c \cup C_1^c) \\
&\ge& 1- P(B_1^c) + P(C_1^c) \\
&\ge& 1- P(B_1^c) + P(D_1^c) \\
&=& 1 - \left(\frac{1}{2}s^2(1 - p_{in}^2)^{s - 2}     + O(p_{out}^2ns(s + \lambda))\right)
\end{eqnarray*}

\begin{eqnarray*}
P(\mbox{correctly clustering $S_*$ with cutting level 2}) &\ge & P(B_2 \cap C_2) \\
&=& 1 - P(B_2^c \cup C_2^c) \\
&\ge& 1- P(B_2^c) + P(C_2^c) \\
&\ge& 1- P(B_2^c) + P(D_2^c) \\
&=& 1 - \left(s^3(1 - p_{in}^2)^{s - 3}     + O(p_{out}^3ns(s+\lambda)^2)\right)
\end{eqnarray*}
\end{proof}

\textbf{Proof of Lemma~\ref{lem1}:} 
\begin{proof}
\begin{eqnarray*}
P(B_1^c) &=& P(\mbox{There exists at least two nodes $i \in S_*$ and $j \in S_*^c$, such that $T_{ij} \geq 1$})\\
&=&P(\bigcup_{i \in S_*}\bigcup_{j \in S_*^c} \bigcup_{k \in S_* \cup S_*^c, k \neq i,j} \{A_{ij}A_{jk}A_{ki}=1\})\\
&\le&sn[P(\bigcup_{k\in S_*} \{A_{ij}A_{jk}A_{ki}=1\}) + P(\bigcup_{k\in S_*^c} \{A_{ij}A_{jk}A_{ki}=1\})]\\
&\le&sn(sp_{out}^2 + np_{out}^2\frac{\lambda}{n})\\
&=& O(p_{out}^2ns(s + \lambda))
\end{eqnarray*}

\begin{eqnarray*}
P(B_2^c) &=& P(\mbox{There exists at least two nodes $i \in S_*$ and $j \in S_*^c$, such that $T_{ij} \geq 2$})\\
&=&P(\bigcup_{i \in S_*}\bigcup_{j \in S_*^c} \bigcup_{k<l \in S_* \cup S_*^c, k,l \neq i,j} \{A_{ij}A_{jk}A_{ki}A_{jl}A_{li} = 1\})\\
&\le& snP(\bigcup_{k,l\in S_*} \cup \bigcup_{k\in S_*,l\in S_*^c} \cup \bigcup_{k,l\in S_*^c} \{A_{ij}A_{jk}A_{ki}A_{jl}A_{li} = 1\})\\
&\le& sn(\frac{1}{2}s^2p_{out}^3 + \frac{1}{2}n^2p_{out}^3{(\frac{\lambda}{n})}^2 + nsp_{out}^3\frac{\lambda}{n})\\
&=& O(p_{out}^3ns(s+\lambda)^2)
\end{eqnarray*}

\begin{eqnarray*}
P(D_\lambda^c) &=& P(\mbox{There exists at least two nodes $i, j\in S_*$, such that $i$ and $j$ has less than $\alpha$ neighbors})\\
&=&P\left(\bigcup_{i,j \in S_*} \{\mbox{$i$ and $j$ has less than $\alpha$ neighbors}\}\right)\\
&\leq& \frac{1}{2}s^2\sum_{k = 0}^{\alpha-1} {s-2 \choose k} {(1 - p_{in}^2)}^{s-2-k}
\end{eqnarray*}
\end{proof}

\textbf{Proof of Theorem~\ref{Lclust}:}
\begin{proof}
In the proof, we assume that $i,j \in S_*$, $P(A_{ij} =  1) = p_{in}$. The proof can be easily extended to the case where $P(A_{ij} =  1) \ge p_{in}$.
First, we prove that $P( \max_{i \in S_*} D_{ii}\ge 2 \mathbb{E}D_{11})$ is well bounded with some non-vanishing probability.
$\forall i \in S_*$, 
\[ \mathbb{E} D_{ii} =  \mathbb{E} D_{11} = \sum_{j \in S_*, j \neq i}  \mathbb{E} A_{ij} + \sum_{j \in S_*^c}  \mathbb{E} A_{ij} = (s-1)p_{in} + \lambda.\] 
$\forall i \in S_*,\ \forall \epsilon >0$, 
\[P(D_{ii} \ge  \mathbb{E} D_{ii} + \epsilon) \le \exp\left\{-\frac{\epsilon}{2( \mathbb{E} D_{ii} + \epsilon/3)}\right\}.\]
Take $\epsilon =  \mathbb{E} D_{11}$, and take union bound for all $i \in S_*$, we have
\[P(\max_{i \in S_*}D_{ii} \ge 2 \mathbb{E} D_{11}) \le s \exp\left\{- \frac{3}{8}  \mathbb{E} D_{11}\right\} = s \exp\{- \frac{3}{8}[(s-1)p_{in} + \lambda]\}.\]
Let $O$ denote the set $\{D_{ii} \le 2 \mathbb{E} D_{11}, \forall i \in S_*\}$. Then within the set $O$, by the same argument from the proof of Theorem 3, we have that with probability at least $\frac{1}{2}s^2(1 - p_{in}^2)^{s-2}$, $S_*$ is clustered within one block by LocalTrans$(L_\tau, i , cut )$ for any $i \in S_*$ with $cut$ 
\[cut=  \left(2 \mathbb{E} D_{11} + \tau\right)^{-3} = \left(2(s-1)p_{in} + 2\lambda+\tau\right)^{-3}.\]
Second part proves that $\forall i \in S_*, j \in S_*^c, P(T_{ij} \ge cut)$ is $o(1)$. 
Notice that for any $j \in S_*^c$, the $(j,j)$th element of $D_\tau$ (denote as $D^\tau_{jj}$) is $d_j^* + \sum_{i \in S_*}A_{ij} + \tau$, so we have
\[D^\tau_{jj} \ge d_j^* + \tau, \quad \forall j \in S_*^c \]
For any $i \in S_*$, we have
\[D^\tau_{ii} \ge \tau, \quad \forall i \in S_*\]
$\forall i \in S_*,\ j \in S_*^c$, 
\begin{eqnarray*}
P(T_{ij} \ge cut) &\le &  \frac{d_j^*}{n}P(T_{ij} \ge\ cut | A_{ij} = 1)\\
&=& \frac{d_j^*}{n}P\left( \frac{1}{D^\tau_{ii}D^\tau_{jj}} \sum_{k = 1}^n \frac{A_{ik}A_{kj}}{D^\tau_{kk}}\ge cut\right) \\
&\le& \frac{d_j^*}{n}P\left( \frac{1}{\tau(d_j^*+\tau)} \sum_{k = 1}^n \frac{A_{ik}A_{kj}}{D^\tau_{kk}}\ge cut\right) \\
&\le&  \frac{d_j^*}{n} \left[P\left( \sum_{k \in S_*} A_{ik}A_{kj}\ge \tau^2(d_j^* + \tau)cut/2\right) + \right. \\
&& \hspace{.5in} \left. P\left( \sum_{k \in S_*^c} \frac{A_{ik}A_{kj}}{d_k + \tau}\ge \tau(d_j^* + \tau)cut/2\right)\right] \\
\end{eqnarray*}

For the first term, when $n$ large, 
\begin{eqnarray*}
P( \sum_{k \in S_*} A_{ik}A_{kj} \ge \tau^2(d_j^* + \tau)cut/2) &\le&  P( \sum_{k \in S_*, k \neq i} A_{ik}A_{kj} > 0) \\
&\le& 1 -  (1 - p_{in}\frac{d_j^*}{n})^{s-1} \\
&\le& p_{in}(s-1)\frac{d_j^*}{n}
\end{eqnarray*}
On the other hand, notice that $\{A_{ik}A_{kj}, k \in S_*/\{i\} \}$ are independent random variables with $A_{ik}A_{kj} \sim Ber(p_{in}\frac{d_j^*}{n})$ by the assumption. $E[\sum_{k \in S_*, k \neq i} A_{ik}A_{kj}] = p_{in}(s-1)\frac{d_j^*}{n}$. By concentration inequality, when $n$ is sufficiently large (independent of $j$), we have 
\begin{eqnarray*}
P ( \sum_{k \in S_*} A_{ik}A_{kj}\ge \tau^2(d_j^* + \tau)cut/2 )& \le &  P\left( \sum_{k \in S_*} A_{ik}A_{kj}\ge 
p_{in}(s-1)\frac{d_j^*}{n}  + \right. \\
&&\hspace{.5in} \left. d_j^*\left(\tau^2cut/2 -\frac{p_{in}(s-1)}{n}\right) \right) \\
&\le& \exp\left(-\frac{c_1^2(d_j^*)^2}{2c_2d_j^*/n + 2c_1d_j^*/3}\right)\\
&=& \exp(-c_1d_j^*)
\end{eqnarray*}
where $c_1 = \tau^2cut/3$.

For the second term, without loss of generality, assume that $A_{1j} = A_{2j} = ... = A_{d_j^*,j} = 1, A_{k,j} = 0, \forall k > d_j^*$.
Notice that  $\{A_{ik}, k = 1,2,..., n,k \neq i \}$ are independent random variables with $A_{ik} \sim Ber(\frac{d_k^*}{n})$. Applying concentration inequality on the sequence $\{A_{ik}, k = 1,2,..., d_j^*\}$, with $a_k = \frac{1}{d^*_k}$. Define $X = \sum_{k = 1}^{d_j^*} a_k A_{ik}$, then $ \mathbb{E} X = \frac{d_j^*}{n}$, and 
\[v = \sum_{k = 1}^{d_j^*} a_k^2  \mathbb{E} A_{ik} =  \frac{1}{n}\sum_{k = 1}^{d_j^*} \frac{1}{d_j^*}\le \frac{d_j^*}{n}.\]
When $n$ is sufficiently large, we have,
\begin{eqnarray*}
P\left( \sum_{k \in S_*^c} \frac{A_{ik}A_{kj}}{d_k + \tau}\ge \tau(d_j^* + \tau)cut/2\right) &\le& P\left( \sum_{k =1}^{d_j^*} \frac{A_{ik}}{d_k}\ge\tau(d_j^* + \tau)cut/2\right)\\
  &\le& P\bigg( \sum_{k =1}^{d_j^*} \frac{A_{ik}}{d_k}\ge \frac{d_j^*}{n} +   d_j^*(\tau cut/2 - \frac{1}{n})\bigg)\\
  &\le& \exp\left( -\frac{c^2(d_j^*)^2}{2(v + cd_j^*/3)}\right)\\
  &\le& \exp\left( -\frac{c^2(d_j^*)^2}{2( d_j^*/n+ cd_j^*/3)}\right)  \\
  &=& \exp(-cd_j^*)
\end{eqnarray*}
where $c = \tau cut/3$.

On the other hand, 
\begin{eqnarray*}
P\left( \sum_{k \in S_*^c} \frac{A_{ik}A_{kj}}{d_k + \tau}\ge \tau(d_j^* + \tau)cut/2\right) &\le & P\left( \sum_{k =1}^{d_j^*} \frac{A_{ik}}{d_k+\tau}\ge  \tau(d_j^* + \tau)cut/2\right) \\
&\le& P\left( \sum_{k=1}^{d_j^*} A_{ik} > 0\right) \\
&\le& 1 -  \left(1 - \frac{d_j^*}{n}\right)^{d_j^*} \\
&\le& \frac{(d_j^*)^2}{n}
\end{eqnarray*}

To sum up,when $n$ is sufficiently large, we have that $\forall i \in S_*,\ j \in S_*^c$
\begin{eqnarray*}
P(T_{ij} \ge cut) \le \frac{d_j^*}{n} \min\left\{\frac{(d_j^*)^2}{n} + p_{in}(s-1)\frac{d_j^*}{n}, \ \ 2e^{-cd_j^*} \right\}.
\end{eqnarray*}
where $cut=  (2 \mathbb{E} D_{11} + \tau)^{-3} = (2(s-1)p_{in} + 2\lambda+\tau)^{-3}$, and $c = \tau  cut/3$.
%We can separate the nodes in $S_*^c$ into two parts $B_M and B_\infty$. Where $B_M = \{j, d_j^* \le M, M$ is some constant$\}$ and
%$B_\infty = \{j, d_j^* = O(n)$ and $d_j^* \rightarrow \infty \}$. Further assume that $0 \le \alpha_n \le 1$ proportion of the nodes belongs to $B_\infty$, that is 
% $|B_M| = (1-\alpha_n)(n-s)$ and  $|B_\infty| = \alpha_n(n-s)$.
%\begin{eqnarray*}
%P(\mbox{$S_*$ and $S_*^c$ are not separated }) &=& P(\cup_{i \in S_*, j \in S_*^c} T_{ij} \ge \frac{1}{\Delta^3})\\
%&\le& \sum_{i \in S_*, j \in S_*^c} P(T_{ij} \ge \frac{1}{\Delta^3})\\
%&\le&  s*\sum_{j \in S_*^c}\frac{d_j^*}{n} min\{\frac{d_j^*^2}{n} + p_{in}s\frac{d_j^*}{n}, 2e^{-cd_j^*} \} \\
%&\le& s*\sum_{j \in B_M} \frac{M}{n}(\frac{M^2}{n} + p_{in}s\frac{M}{n}) + s*\sum_{j \in B_\infty} \frac{d_j^*}{n}*2e^{-cd_j^*}\\
%&\le& s(1-\alpha_n)M(\frac{M^2}{n} + p_{in}s\frac{M}{n}) + s*\alpha_n* max_{j \in B_\infty} (d_j^**2e^{-cd_j^*})\\
%&\le& o(1)
%\end{eqnarray*}
%If $\alpha_n = O(n^{-1})$, then it is bounded by $O(n^{-1})$. Or we can make further assumptions on $d_j^*$ to make the bound better.

%The equation 
%\begin{eqnarray*}
%P(T_{ij} \ge  cut) \le \frac{d_j^*}{n} min\{\frac{(d_j^*)^2}{n} + p_{in}s\frac{d_j^*}{n}, 2e^{-cd_j^*} \}.
%\end{eqnarray*}
%is the fundamental quantity.  We will eventually need to multiply by $s$ and sum over $d_j^*$.   So, if $n$ times the above equation converges to zero (for any $d_j^*$), then we are ok. 
Next we show that, for any $i \in S_*, j \notin S_*$, $ P(T_{ij} \ge cut) = O(n^{3\epsilon-2} )$, where $\epsilon > \log_n(\frac{1}{c})$.

\textbf{Case 1:}  $d_j^* < n^\epsilon$. 
\begin{eqnarray*}
n P(T_{ij} \ge cut) &\le& \min\left\{\frac{(d_j^*)^3}{n} + p_{in}s\frac{(d_j^*)^2}{n}, \  2d_j^*e^{-cd_j^*} \right\} \\
&\le& \frac{(d_j^*)^3}{n} + p_{in}s\frac{(d_j^*)^2}{n}\\
&\le &\frac{n^{3\epsilon}}{n} + p_{in}s\frac{n^{2\epsilon}}{n}\\
&=& O(n^{3\epsilon - 1})
\end{eqnarray*}

\textbf{Case 2:} $d_j^* \ge  n^\epsilon$. Notice that the derivate of $u(x) = 2d_j^*\exp(-cd_j^*)$ is $(1-cx)\exp(-cx)$. So, if $n^\epsilon>1/c$ then $u(d_j^*)\le u(n^\epsilon) = 2n^\epsilon \exp(-cn^\epsilon)$.  
\begin{eqnarray*}
n P(T_{ij} \ge cut) &\le& \min\left\{\frac{(d_j^*)^3}{n} + p_{in}s\frac{(d_j^*)^2}{n}, \  2d_j^*\exp(-cd_j^*) \right\} \\
&\le& 2d_j^*\exp(-cd_j^*)\\
&\le& 2n^\epsilon \exp(-cn^\epsilon)\\
&=& o(n^{-1})
\end{eqnarray*}
So, independent of $d_j^*$, 
\[P(T_{ij} \ge cut) = O(n^{3\epsilon - 2})\]
Putting the pieces together,
\begin{eqnarray*}
P(\mbox{$S_*$ and $S_*^c$ are not separated }) &\le&  s \sum_{j \in S_*^c}\frac{d_j^*}{n} \min\left\{\frac{(d_j^*)^2}{n} + p_{in}s\frac{d_j^*}{n}, \ 2e^{-cd_j^*} \right\} \\
&\le& s n O(n^{3\epsilon - 2})\\
&=& O(n^{3\epsilon - 1})
\end{eqnarray*}
%So, we need $\epsilon <1/3$ and 
%\[\epsilon> -\frac{\log(c)}{\log(n)}= \log_n(\frac{1}{c}).\]
Finally, recall that  $O = \{D_{ii} \le 2 \mathbb{E} D_{11}, \forall i \in S_*\}$,  we have that for any $i \in S_*$,  
\begin{eqnarray*}
P(\{\mbox{LocalTrans$(L_\tau, i, cut)$ returns $S_*$}\}) &\ge& 1 - \frac{1}{2}s^2(1 - p_{in}^2)^{s-2} - O(n^{3\epsilon - 1}) - P(O^c)\\
&\ge&  1 - \bigg(\frac{1}{2}s^2(1 - p_{in}^2)^{s-2} +  \\      
&& \hspace{.2in}  s \exp\{- \frac{3}{8}[(s-1)p_{in} + \lambda)\} + O(n^{3\epsilon - 1})\bigg).
\end{eqnarray*}
\end{proof}

\end{appendix}

\end{document}